\documentclass{article}

\usepackage[accepted]{icml2026}  

\makeatletter
\def\ICML@appearing{Preprint. Under review.}
\makeatother

\usepackage{microtype}
\usepackage{graphicx}
\usepackage{subcaption}
\usepackage{booktabs}
\usepackage{titletoc}  

\usepackage{mathtools}
\usepackage{amsmath,amsfonts,amssymb}
\usepackage{amsthm}
\usepackage{xcolor}

\usepackage[colorlinks=true,linkcolor=blue,citecolor=blue,urlcolor=blue]{hyperref}

\usepackage[capitalize,noabbrev]{cleveref}

\renewcommand{\algorithmicrequire}{\textbf{Input:}}
\renewcommand{\algorithmicensure}{\textbf{Output:}}

\usepackage[textsize=tiny]{todonotes}

\newtheorem{theorem}{Theorem}
\newtheorem{corollary}{Corollary}
\newtheorem{prop}{Proposition}
\newtheorem{remark}{Remark}
\newtheorem{lemma}{Lemma}

\newtheorem{definition}{Definition}

\newcommand{\bb}{\mathbf{b}}
\newcommand{\bc}{\mathbf{c}}
\newcommand{\bd}{\mathbf{d}}
\newcommand{\be}{\mathbf{e}}
\newcommand{\bbf}{\mathbf{f}}
\newcommand{\bg}{\mathbf{g}}
\newcommand{\bh}{\mathbf{h}}

\newcommand{\bp}{\mathbf{p}}

\newcommand{\bs}{\mathbf{s}}

\newcommand{\bu}{\mathbf{u}}
\newcommand{\bv}{\mathbf{v}}
\newcommand{\bw}{\mathbf{w}}
\newcommand{\bx}{\mathbf{x}}
\newcommand{\by}{\mathbf{y}}

\newcommand{\bA}{\mathbf{A}}

\newcommand{\bD}{\mathbf{D}}

\newcommand{\bI}{\mathbf{I}}

\newcommand{\bQ}{\mathbf{Q}}

\newcommand{\bSigma}{\boldsymbol{\Sigma}}
\newcommand{\bxi}{\boldsymbol{\xi}}

\icmltitlerunning{Minimizing Collateral Damage in Activation Steering}

\begin{document}

\twocolumn[
\icmltitle{Minimizing Collateral Damage in Activation Steering}

\begin{icmlauthorlist}
\icmlauthor{Tam Nguyen}{rice}
\icmlauthor{Tu Anh Nguyen}{ricemath}
\icmlauthor{Sina Alemohammad}{ut}
\icmlauthor{Richard G. Baraniuk}{rice}
\end{icmlauthorlist}

\icmlaffiliation{rice}{Department of Electrical \& Computer Engineering, Rice University, Houston, USA}
\icmlaffiliation{ricemath}{Department of Computational and Applied Mathematics, Rice University, Houston, USA}
\icmlaffiliation{ut}{Department of Electrical \& Computer Engineering, The University of Texas at Austin, Austin, USA}

\icmlcorrespondingauthor{Tam Nguyen}{mn72@rice.edu}

\icmlkeywords{Activation Steering, Large Language Models, Representation Geometry}

\vskip 0.3in
]

\printAffiliationsAndNotice{}

\begin{abstract}
Activation steering is a method for controlling Large Language Model (LLM) behavior by intervening in its internal representations to increase the alignment with a specific target feature direction. However, standard interventions, such as vector addition, often cause ``collateral damage", defined as unintended changes in the alignment of activations along other non-target feature directions. This damage occurs because standard methods implicitly assume the isotropy of non-target features. In this work, we provide a mathematical formalization of collateral damage and introduce a principled framework that models steering as a constrained optimization problem. Our method finds a new activation that minimizes the expected squared collateral change weighted by the empirical second-moment matrix of activations. This weighting encodes the nonuniform cost of the perturbation in different feature directions, in contrast to isotropic approaches that penalize changes uniformly in all feature directions. By accounting for the empirical second-moment of activations, our approach achieves more precise control while reducing the degradation of model performance on unrelated tasks.
\end{abstract}

\section{Introduction} 
\label{sec:introduction}

Controlling LLM behavior has traditionally relied on retraining or fine-tuning~\cite{zhang2026survey, ziegler2019fine, Dathathri2020Plug}, processes that become increasingly expensive at scale. {\em Activation steering} provides a more efficient alternative, enabling behavioral control at inference time without modifying the model weights~\cite{zou2023repeng, li2024ITI, panickssery2023steering_llama2, liu2023context}. By directly intervening in internal representations, this ``plug-and-play'' approach nudges activation vectors to enhance desirable traits like factuality~\cite{li2024ITI} or to suppress negative tendencies such as sycophancy, toxicity, and hallucination~\cite{chen2025persona, zhang-etal-2024-truthx}. Recent work has highlighted the broad applicability of activation steering in managing diverse semantic traits—ranging from appropriateness to humor and creativity~\cite{wang2025persona}—demonstrating that alignment can be achieved without costly parameter updates.

A common method for activation steering is {\em ActAdd}~\cite{turner2023actadd}, which adds a scaled steering vector directly to a model's activations. While simple to implement, this class of methods can be scale-sensitive, since the magnitude of the steering vector can significantly alter the model's internal distribution and lead to performance degradation. {\em Norm-Preserving Steering} methods address this by modifying only the activation direction while keeping the original magnitude intact to avoid disrupting the model's internal scaling~\cite{pham-nguyen-2024-householder}. Examples include Slerp (for Spherical linear interpolation)~\cite{Shoemake1985AnimatingRW, white2016sampling, jang2024spherical}, which rotates the activation vector along the great circle towards the target direction until the desired cosine similarity is met. 
Similarly, {\em Angular Steering}~\cite{vu2025angular} projects the activation onto a fixed 2D plane and rotates it to a target angle, changing only the in-plane component while preserving the global norm.

Existing steering methods face a {\em trade-off between steering effectiveness and downstream performance}. 
For example, the experiment reported in Figure~\ref{fig:model-comparison-pareto}(a) shows that achieving a 30-40\% steering success rate with ActAdd causes a 10–20\% drop in performance in Gemma-2-9B-IT~\cite{team2024gemma}. 
Meanwhile, Angular Steering maintains downstream performance but fails to steer effectively, reaching a success rate of less than 10\% in the same model.

The steering/performance trade-off arises because steering the model to exhibit one feature can alter the alignment of the activations with other features. 
Existing steering methods 
implicitly treat an activation shift in any direction as equally costly (isotropic). 
In reality, the geometry of LLM representations is {\em anisotropic}. 
The features are not uniform; some directions are more coupled to others. 
And when the features are entangled with the steering direction, ignoring this geometry can unintentionally disturb them.



In addition while the general trade-off between steering success and performance has been observed~\cite{zhang-etal-2025-safety, zhao-etal-2025-adasteer, yousefpour-etal-2025-representation}, existing work has focused on heuristic monitoring of downstream benchmarks that address the symptoms of performance loss rather than the geometric root causes.

In this paper, we take initial steps toward understanding and mitigating the tradeoff between activation steering success and model performance. We formalize the concept of \textit{collateral damage} as the expected unintended alignment shift toward non-target directions and empirically show that \textit{increased collateral damage leads to performance degradation} (see Fig.~\ref{fig:damage_correlation}). Motivated by a geometric view of representation space, we propose a general framework for minimizing collateral damage during steering. Geometrically, given a fixed alignment budget—defined as the cosine similarity between the activation and the target direction— we find the safest steered activation that lies on a particular “latitude” of the hypersphere of normalized activations.
\begin{figure}[t]
    \centering
    \includegraphics[width=\linewidth]{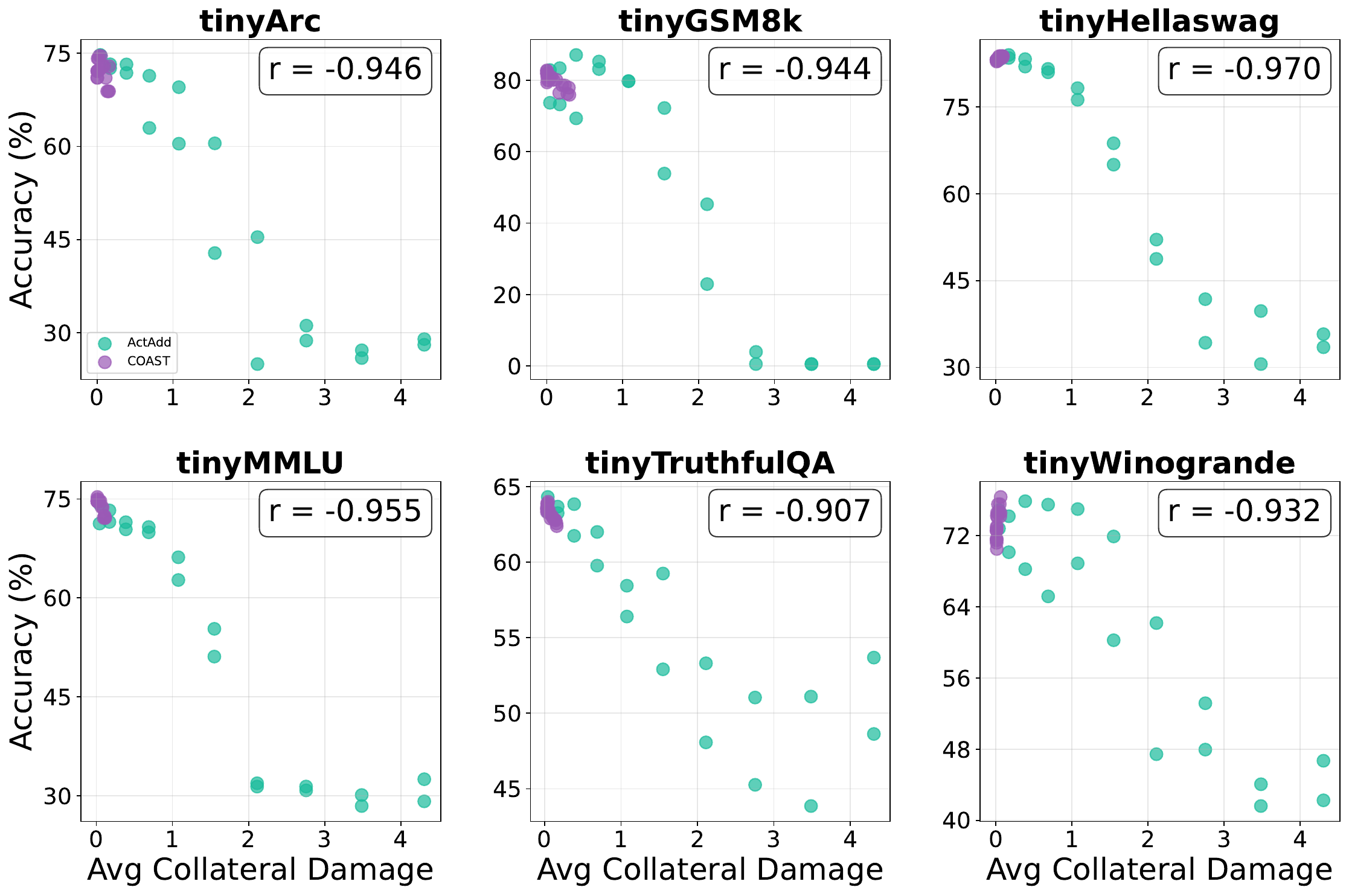}
    
    \caption{\small Negative correlation between {\em collateral damage} and {\em accuracy} of steering models on Qwen2.5-14B-Instruct. Across six benchmarks, we observe a strong negative Pearson correlation ($r < -0.9$) between the average collateral damage and the model's accuracy. This validates that our collateral damage metric is a reliable proxy for performance degradation.
    }
    \label{fig:damage_correlation}
\end{figure}
Within this framework, existing activation steering approaches correspond to a special case that we term \textit{isotropic collateral damage}, where the degradation incurred by steering is assumed to be invariant to the direction of the activation shift and depends only on its magnitude. However, LLM representations are inherently anisotropic; consequently, true collateral damage is direction-dependent and highly sensitive to the geometry of the representation space. This observation motivates explicitly accounting for anisotropy when designing steering methods, which is central to our approach.



We introduce {\em COllateral-damage Minimizing Activation STeering} (COAST), which applies algorithms rooted in differential geometry to traverse the solution space and find the optimal steering vector. We prove that COAST converges to the global minimum of the collateral damage objective in the continuous limit. Furthermore, it strictly generalizes SLERP, the norm-preserving version of ActAdd. 
COAST reduces to SLERP under isotropic conditions but adapts to protect performance when anisotropy is present. Like other additive steering methods, COAST is plug-and-play, adding only a small overhead to the inference cost (e.g., 0.9\% to 4\% in our experiments).


While baseline methods like ActAdd and Angular Steering exhibit a trade-off between steering effectiveness and model performance—especially in reasoning-intensive tasks included in tinyBenchmarks~\cite{polo2024tinybenchmarks}—our experiments demonstrate that COAST effectively mitigates this compromise. Indeed, {\em COAST achieves high-quality steering without degrading performance}, yielding up to a 30\% improvement in Attack Success Rate (ASR) 
over Angular Steering and improving up to 20\% in accuracy compared to ActAdd at the same ASR. In addition, COAST outperforms SLERP in steering success while preserving the model's performance.


By explicitly minimizing collateral damage under a fixed steering budget, COAST makes it possible to push models {\em where you want them to go without breaking what already works}. This shifts the conversation from ``can we steer?'' to ``can we steer responsibly,'' opening the door to deployment settings where reliability, safety, and downstream performance must all be simultaneously protected.

\section{Related Work} 
\label{sec:related_works}

\textbf{Activation steering.} A growing line of research controls LLM behavior by intervening directly on intermediate activations at inference time~\cite{zou2023repeng, li2024ITI, panickssery2023steering_llama2, liu2023context}.~\cite{turner2023actadd} formalizes this paradigm as {\em activation engineering} and introduces Activation Addition (ActAdd). This method involves computing a steering direction—typically derived from contrastive prompt pairs—and then injecting a scaled version of this vector into the residual stream to shift model attributes. Variants of this approach have been employed to steer behaviors such as truthfulness~\cite{li2024ITI, zhang-etal-2024-truthx} and to modulate persona-style traits like sycophancy, toxicity, and hallucination propensity~\cite{chen2025persona}. Similarly,~\cite{wang2025persona} applies inference-time steering to control high-level qualities like appropriateness, humor, and creativity. However, a critical drawback of these methods is {\em scale sensitivity}: steering quality relies heavily on the intervention coefficient and the effective steering vector norm. This requires careful hyperparameter tuning, and overly strong intervention can degrade fluency or introduce unintended side effects.


\textbf{Norm-constrained activation steering.}~\cite{arditi2024refusal} introduces directional ablation, which removes the steering vector from the residual stream. By normalizing the steering vector before removal, this method partially mitigates scale sensitivity, though it lacks flexibility as it can only erase the direction, not amplify or attenuate it. Conversely,~\cite{vu2025angular} introduce Angular Steering to promote norm-preservation and flexibility. This method constructs a 2D plane using the steering vector and the first principal component of the steering vectors across all layers, performing rotation on this plane while preserving the activation components orthogonal to it. However, because the rotation is restricted to a fixed 2D plane, Angular Steering struggles to satisfy specific alignment budgets, can leads to struggle in steering success. Furthermore, Angular Steering relies on the assumption that features are nearly orthogonal. In Section~\ref{sec:coast_to_slerp}, we demonstrate that, when this assumption holds, Spherical Linear Interpolation (Slerp) is the optimal solution for minimizing collateral damage. Generalizing beyond this constraint, COAST achieves the global optimum solution given an alignment budget.

\textbf{Robust activation steering.} Prior work has highlighted a trade-off between safety and utility, as activation steering can cause over-refusal and degrade model usefulness~\cite{zhang-etal-2025-safety, zhao-etal-2025-adasteer, yousefpour-etal-2025-representation}. Mitigation strategies include conditional activation steering~\cite{lee2025programming}, sharpening target direction~\cite{wang2025surgical, shen2025jailbreak}, or nullspace-constrained approaches~\cite{Wang2024SteeringAF, sheng2025alphasteer} that dynamically induce refusal only for malicious prompts. Our work generalizes the goal of activation steering beyond safety to the broader challenge of altering any attribute without compromising overall performance. 
Our approach is orthogonal to these robust methods focusing explicitly on formulating and minimizing collateral damage during steering.

\section{Collateral-damage Minimizing Activation Steering (COAST)}

In this paper we study normalized activation steering: we edit a representation while preserving its magnitude, and we choose the intervention to minimize collateral damage subject to a directional alignment budget. The steering direction serves as a simple control knob, i.e., an explicit constraint on how far the intervention may deviate from the intended steering direction. Concretely, we enforce this budget by lower-bounding the cosine similarity between the intervention vector and the steering direction. This formulation naturally treats features as directions in activation space, aligning with the linear feature hypothesis~\cite{park2024linearhypothesis, Nanda2023EmergentLR, huben2024sparse, jiang2024origins} and enabling principled, geometry-aware steering.
\subsection{Steering with an alignment budget}
\label{subsec:alignment_budget_steering}
Let $\mathbf{h} \in \mathbb{R}^p$ denote the activation, and let $\mathbf{d} \in \mathbb{R}^p$ denote a target feature direction. We seek a new activation $\mathbf{x} \in \mathbb{R}^p$ that satisfies a prescribed alignment budget $\mathbf{d}^\top \mathbf{x} = \alpha$,
where $\alpha \in [-1, 1]$ is the desired post-steering cosine similarity with $\mathbf{d}$. We also enforce norm preservation by working on the unit sphere and assuming that $\|\mathbf{h}\|=\|\mathbf{d}\|=1$ and $\|\mathbf{x}\|=1$. This normalization is without loss of generality, since we can rescale the steered activation by its original activation norm after the intervention.

We begin by formulating an objective that explicitly quantifies the impact of steering on non-target features.
To this end, we model non-target features as belong to the family $\mathcal{F} = \{ \mathbf{f} \in \mathbb{R}^p \mid \mathbf{f} \neq \mathbf{d}, \text{ } \mathbf{f} \text{ is semantic} \}$, representing all meaningful directions excluding the target.
We measure the collateral change induced by steering as the change in alignment along $\mathbf{f}$, given by $\mathbf{f}^\top(\mathbf{x}-\mathbf{h})$. We define the {\em collateral damage} as the expected squared collateral change over the population of non-target feature directions
\begin{equation}
\mathbb{E}\big[(\mathbf{f}^\top(\mathbf{x}-\mathbf{h}))^2\big]
\;=\; (\mathbf{x}-\mathbf{h})^\top \mathbf{\Sigma}_{\mathbf{f}} (\mathbf{x}-\mathbf{h}),
\label{eq:cd}
\end{equation}
where the expectation is summarized by the second moment matrix $\mathbf{\Sigma}_{\mathbf{f}} := \mathbb{E}[\mathbf{f}\mathbf{f}^\top]$ that we dub the {\em collateral weighting matrix}.

Combining the fixed-norm constraint and the alignment budget defines the feasible set
\begin{equation}
\label{eq:constraint}
    \mathcal{M} \;:=\; \{\mathbf{x} \in \mathbb{R}^p : \|\mathbf{x}\|=1,\; \mathbf{d}^\top \mathbf{x} = \alpha\}.
\end{equation}
Our steering problem is therefore the constrained optimization
\begin{equation}
\min_{\mathbf{x} \in \mathcal{M}} (\mathbf{x}-\mathbf{h})^\top \mathbf{\Sigma}_{\mathbf{f}} (\mathbf{x}-\mathbf{h}).
\label{eq:main_problem}
\end{equation}
The solution to \eqref{eq:main_problem} achieves the required steering effect with minimal collateral damage.
\begin{figure}[!t]
\centering
\includegraphics[width=0.28\textwidth]{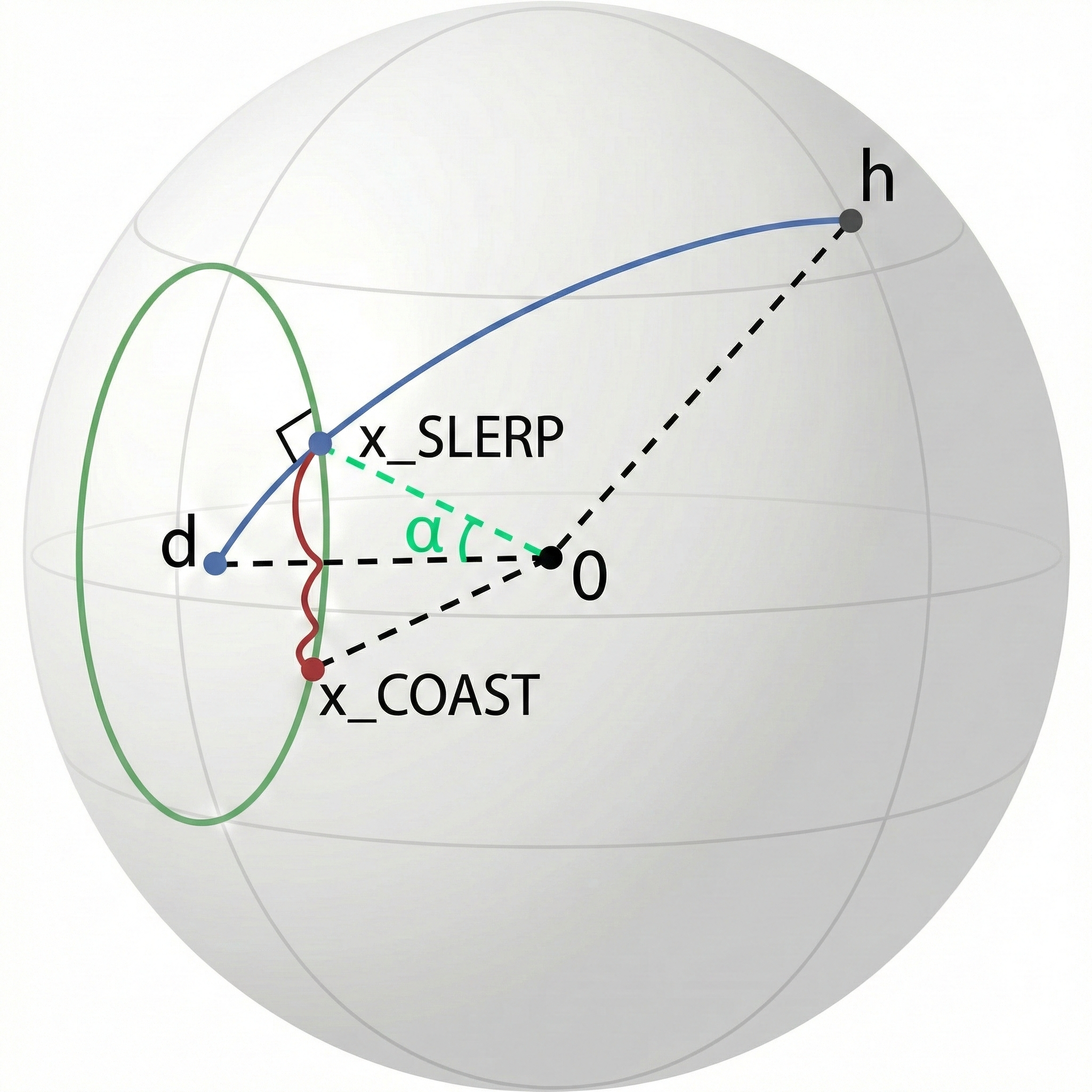}
\vspace{0.1in}
\caption{\small Unlike the rigid Slerp path (blue), our optimized COAST trajectory (red) traverses the manifold of valid steering vectors (green) in order to minimize the collateral damage.
}
\label{fig:COAST} 
\end{figure}
\subsection{Choosing the collateral weighting matrix $\mathbf{\Sigma}$}
\label{subsec:choose_sigma}
The collateral matrix $\mathbf{\Sigma}_{\mathbf{f}} = \mathbb{E}[\mathbf{f}\mathbf{f}^\top]$ specifies which feature directions are costly to perturb. Specifically, $\mathbf{\Sigma}_{\mathbf{f}}$ treats feature directions as uniformly important—meaning a change along feature $f_i$ is just as costly as a change along feature $f_j$.

However, it is crucial to note that treating features with equal importance does not imply that the optimization landscape is isotropic (i.e., that steering costs are the same in all ambient directions). For example, if features are clustered—as is common in LLMs—steering toward a cluster disturbs many features simultaneously, creating high-cost directions. To clarify this distinction, we introduce the following.

\begin{remark}[Two notions of uniformity]
\label{remark:uniformity}
We distinguish between two distinct types of uniformity that govern the behavior of activation steering.

{\em Uniform importance} relates to our prescriptive intent; it assumes that every non-target feature direction should be penalized equally. This leads to an unweighted objective where the collateral matrix $\Sigma_{f}$ considers all directions to be of equal functional concern.

{\em Uniform geometric distribution} relates to the descriptive property of the latent space; it assumes that feature directions themselves are distributed uniformly (isotropically) on the unit sphere. 
\end{remark}
In practice, applications often require feature-preservation prioritization that is task-dependent or user-specific. Therefore, we generalize the notion of $\mathbf{\Sigma}$ to an abstract collateral matrix. Different choices of $\mathbf{\Sigma}$ encode different notions of feature preservation and can reflect task-dependent or user-specific requirements.

We generalize the unweighted expectation by introducing a weighted objective that incorporates a feature importance function $w(\mathbf{f}) \geq 0$.
The resulting collateral objective is defined as
$\mathbb{E}[w(\mathbf{f})(\mathbf{f}^\top (\mathbf{x} - \mathbf{h}))^2] = (\mathbf{x} - \mathbf{h})^\top \mathbf{\Sigma}_w (\mathbf{x} - \mathbf{h})$, 
where the collateral weighting matrix is
$\mathbf{\Sigma}_w := \mathbb{E}[w(\mathbf{f}) \mathbf{f}\mathbf{f}^\top]$. 
Unlike standard isotropic methods that assign equal weight to all perturbations, this formulation enables explicit collateral preservation based on specific functional preferences. By adjusting $w(\mathbf{f})$, practitioners can prioritize the protection of certain semantic directions over others, moving beyond the limiting assumption of uniform significance across the feature landscape

The primary optimization challenge is that the latent feature family $\mathcal{F}$ is not enumerable, making direct computation of $\boldsymbol{\Sigma}_w$ infeasible. A natural approximation is to represent features using a learned dictionary $\mathbf{F} = [\mathbf{f}_1, \dots, \mathbf{f}_m]$ via Sparse Autoencoders (SAEs)~\cite{ng2011sparse, gao2025scaling, lieberum-etal-2024-gemma}. To quantify feature importance, one approach is utilize a distribution of reference activations, $\mathcal{H}_{\text{ref}}$, collected by passing a general text corpus through the model. Since SAEs decompose any activation $\mathbf{h} \in \mathcal{H}_{\text{ref}}$ as $\mathbf{h} \approx \sum_i c_i(\mathbf{h}) \mathbf{f}_i$, we can define the importance weight $w_i = \mathbb{E}_{\mathbf{h} \sim \mathcal{H}_{\text{ref}}}[c_i(\mathbf{h})^2]$. This formulation ensures that features used frequently or strongly across the general data distribution are penalized more heavily. This yields the dictionary-based collateral weighting matrix $\boldsymbol{\Sigma}_{w} = \sum_i w_i \mathbf{f}_i \mathbf{f}_i^\top$. While interpretable, this construction introduces substantial overhead, as it requires training high-quality SAEs for every model and layer of interest.


More fundamentally, the dictionary-based approach is limited because it ignores feature interactions. By summing over individual-feature moments, it implicitly assumes that features are statistically independent, thereby treating interaction terms as zero. In reality, however, features can co-activate~\cite{deng2025sparse}. To respect the feature geometry without relying on explicit decomposition, we propose using the empirical second moment of the activations, $\boldsymbol{\Sigma}_{\mathbf{h_{\text{ref}}}} = \mathbb{E}_{\mathbf{h} \sim \mathcal{H}_{\text{ref}}}[\mathbf{h_{\text{ref}}}\mathbf{h_{\text{ref}}}^\top]$, as the collateral weighting matrix. Unlike the dictionary-based approximation, this metric rigorously aggregates both individual feature moments and their co-activation terms:
\begin{equation*}
\boldsymbol{\Sigma}_{\mathbf{h_{\text{ref}}}} = \mathbb{E}[\mathbf{h_{\text{ref}}}\mathbf{h_{\text{ref}}}^T] \approx \mathbb{E} \left[ \sum_i c_i^2 \mathbf{f}_i \mathbf{f}_i^T + \sum_{i \neq j} c_i c_j \mathbf{f}_i \mathbf{f}_j^T \right],
\end{equation*}
where $c_i = c_i(\bh_{\text{ref}})$, and $c_j = c_j(\bh_{\text{ref}})$.
By adopting $\mathbf{\Sigma}_{\mathbf{h_{\text{ref}}}}$, COAST can account for anisotropic feature correlations and co-activations without requiring explicit feature decomposition or external dictionaries.
In practice, the empirical second moment $\boldsymbol{\Sigma}_{\mathbf{h_{\text{ref}}}}$ can be computed using $N$ tokens sampled from a general-purpose corpus or a domain-specific dataset relevant to the target task.


Henceforth, we will use $\mathbf{\Sigma}$ to denote a general collateral weighting matrix; we will specify $\mathbf{\Sigma}_{\mathbf{f}}$, $\mathbf{\Sigma}_{w}$, or $\boldsymbol{\Sigma}_{\mathbf{h_{\text{ref}}}}$ only when the distinction is necessary.

\subsection{COAST optimization on the budget sphere $\mathcal{M}$}
\label{subsec:COAST-opt}

We now develop an algorithm to efficiently solve the constrained optimization problem $\min_{\bx \in \mathcal{M}} J(\bx)$, where $J(\bx) = (\bx-\bh)^\top \bSigma(\bx-\bh)$ is the collateral damage and $\mathcal{M}$ is the feasible set from \eqref{eq:constraint}.

\textbf{Geometry of the feasible set.}
The budget constraint $\bd^\top \bx=\alpha$ fixes a {\em target latitude} on the unit sphere.
Consequently, the feasible set $\mathcal{M}$ defined in \eqref{eq:constraint} forms a $(p-2)$-sphere with center $\alpha\bd$ and radius $r=\sqrt{1-\alpha^2}$. 
This geometric perspective naturally motivates the use of Riemannian optimization.
\subsubsection{Tangent space and Riemannian gradient on $\mathcal{M}$}
To adopt techniques from Riemannian optimization, we first need to characterize the tangent space $T_{\bx}\mathcal{M}$ for any $\bx \in \mathcal{M}$.
A vector $\bxi \in \mathbb{R}^P$ is tangent to the manifold $\mathcal{M}$ at $\bx$ if an infinitesimal step from $\bx$
along a curve on $\mathcal{M}$ with initial velocity $\bxi$ does not violate the optimization constraints. 
This implies $\bxi$ must be orthogonal to the gradients of the constraint functions
$c_1(\bx) \coloneqq \tfrac12(\|\bx\|^2-1)$ and $c_2(\bx) \coloneqq \bd^\top\bx-\alpha$.
Since $\nabla c_1(\bx)=\bx$ and $\nabla c_2(\bx)=\bd$, we have
\begin{equation}
    T_{\bx}\mathcal{M} = \{\bxi \in \mathbb{R}^P : \bx^\top \bxi = 0,\ \bd^\top \bxi = 0\}.
\end{equation}
\textbf{Tangent-space projector.}
Let $\Pi_{\bx}$ denote the orthogonal projector onto $T_{\bx}\mathcal{M}$.
Since the normal space of $T_{\bx}\mathcal{M}$ is the span of $\{\bx,\bd\}$, 
the unique symmetric projector with range $T_{\bx}\mathcal{M}$ is given by
\begin{equation}
\label{eq:proj}
\nonumber
\Pi_{\bx} = \bI - \bx\bx^\top - \frac{(\bd-\alpha\bx)(\bd-\alpha\bx)^\top}{1-\alpha^2}.
\end{equation}
Here, $\bd-\alpha\bx$ is the component of $\bd$ orthogonal to $\bx$ and
$\|\bd-\alpha\bx\|^2 = 1-\alpha^2 = r^2$.

\textbf{Riemannian gradient.} The Euclidean gradient of $J(\bx)$ is
$\nabla J(\bx) = 2\bSigma(\bx-\bh)$.
Projecting $\nabla J(\bx)$ onto the tangent space yields the Riemannian gradient, we derive the steepest-ascent direction on $\mathcal{M}$ as
\begin{equation}
\label{eq:riemgrad}
\mathrm{grad} J(\bx) = \Pi_{\bx}\nabla J(\bx) = \Pi_{\bx}\big(2\bSigma(\bx-\bh)\big).
\end{equation}
Based on \eqref{eq:riemgrad}, the underlying idea of the COAST optimization algorithm in Alg.~\ref{alg:geodesic_descent} is to minimize $J$ by descending along $-\mathrm{grad},J(\bx)$ while remaining on the feasible set $\mathcal{M}$.

\subsubsection{Geodesic update via the exponential map}
Given $\bx \in \mathcal{M}$, we take a tangent step $\bv = -\eta\,\mathrm{grad} J(\bx) \in T_{\bx}\mathcal{M}$ and update
by following the manifold geodesic defined by the exponential map $\exp_{\bx}(\bv)$.
Because $\mathcal{M}$ is a sphere of radius $r$ centered at $\alpha\bd$, this map has a closed form as expressed in Lemma~\ref{lem:exp-map}.

\begin{lemma}[Exponential map on $\mathcal{M}$]
\label{lem:exp-map}
Let $\bx \in \mathcal{M}$ and $\bv \in T_{\bx}\mathcal{M}$. The geodesic step from $\bx$ in direction $\bv$ is
\begin{equation}
\label{eq:exp}
\exp_{\bx}(\bv) \;=\; \alpha\bd + (\bx-\alpha\bd)\cos(\tau) \;+\; r\,\frac{\bv}{\|\bv\|}\sin(\tau),
\end{equation}
where $\tau = \|\bv\|/r$ and $r=\sqrt{1-\alpha^2}$.
\end{lemma}
The proof is provided in Appendix~\ref{subsec:proof-thm1}.

\textbf{Optimization algorithm.}
Essentially, Lemma~\ref{lem:exp-map} provides a mechanism for moving in the direction of $-\mathrm{grad}J(\bx)$ while remaining on the $(p-2)$-sphere $\mathcal{M}$. This completes the final component of COAST.
As we see in Algorithm~\ref{alg:geodesic_descent}, COAST first initializes to a point $\bx_0 \in \mathcal{M}$ and then iteratively performs geodesic updates along the negative Riemannian gradient.

\begin{algorithm}[t]
\caption{COllateral-damage Minimizing Activation STeering (COAST)}
\label{alg:geodesic_descent}
\begin{algorithmic}
\REQUIRE Initial activation $\bh$, target direction $\bd$, second-moment $\bSigma$,
alignment budget $\alpha$, learning rate $\eta$, max iterations $T$.
\STATE {\bf Precompute:} $r \leftarrow \sqrt{1-\alpha^2}$
\STATE {\bf Initialize:}
  \STATE $\bu_0 \leftarrow \dfrac{\bh - (\bd^\top \bh)\bd}{\|\bh - (\bd^\top \bh)\bd\|};$ \quad $\bx_0 \leftarrow \alpha\bd + r\bu_0$

\FOR{$t = 0$ {\bf to} $T-1$}
  \STATE $\bg_t \leftarrow 2\bSigma(\bx_t-\bh)$; \quad $\bp_t \leftarrow \bd - \alpha \bx_t$

  \STATE $\bxi_t \leftarrow -\Big(\bg_t - (\bx_t^\top \bg_t)\bx_t\Big) + \dfrac{\bp_t^\top \bg_t}{r^2}\bp_t$ 


  \STATE $\bv_t \leftarrow \bxi_t/\|\bxi_t\|;$ \quad $\tau \leftarrow \eta\,\|\bxi_t\|/r$ 
  \STATE $\bx_{t+1} \leftarrow \alpha\bd + (\bx_t-\alpha\bd)\cos\tau + r\,\bv_t\sin\tau$ 
\ENDFOR
\ENSURE $\bx_T$
\end{algorithmic}
\end{algorithm}

\begin{remark}[Geometric interpretation: Safe landing in target latitude]
Geometrically, the alignment constraint $d^{\top}x = \alpha$ restricts the feasible set $\mathcal{M}$ to a specific ``latitude'' on the $(p-2)$-sphere. Our optimization algorithm interprets this latitude as a standalone manifold and searches for the ``safest landing spot'' for the steered activation $x$. Crucially, by utilizing intrinsic geodesic steps, COAST ensures that the prescribed alignment budget is satisfied exactly at every iteration while simultaneously minimizing the collateral damage. The collateral weighting matrix $\Sigma$ acts as a Riemannian metric that deforms the search space; it induces a ``functional topography'' of ridges and valleys, where movement into sensitive, high-variance feature directions is penalized, effectively guiding the steered activation toward the safest available representation.
%
\end{remark}

\textbf{Alternative solver via KKT}.
We refer interested readers to Appendix~\ref{sec:kkt-solution} for details on solving the Karush–Kuhn–Tucker (KKT) system associated with problem~\eqref{eq:main_problem}. In brief, the KKT conditions reduce to a single-variable root-finding problem. By exhaustively enumerating all roots, one can guarantee the global optimum of~\eqref{eq:main_problem}. While this approach provides valuable theoretical insights, its practical applicability is limited. As demonstrated in our experiments in Section~\ref{sec:bottleneck}, an approximate solution to~\eqref{eq:main_problem} is sufficient and enables significantly faster computation.

\subsection{COAST/Slerp equivalence}
\label{sec:coast_to_slerp}
While COAST is a strict generalization of Slerp, when the collateral damage (\ref{eq:cd}) is isotropic, COAST reduces to Slerp and so attains the same performance.
COAST provides a unified framework: it recovers Slerp, the norm-preserving counterpart of ActAdd, in the isotropic collateral damage case, yet it correctly adapts the steering solution when collateral damage is anisotropic.

\begin{definition}[Isotropic Collateral Damage (ICD)]
\label{def:isotropic_CD}
    The collateral damage $J(\boldsymbol{\Delta}) = \boldsymbol{\Delta}^T\boldsymbol{\Sigma}\boldsymbol{\Delta}$ from (\ref{eq:cd}), where $\boldsymbol{\Delta} := \bx - \bh$, is {\em isotropic} if it is invariant to the direction of $\Delta$, i.e.,
    $
    J(\boldsymbol{\Delta}_1) = J(\boldsymbol{\Delta}_2) \quad \text{if} \quad \|\boldsymbol{\Delta}_1\|_2 = \|\boldsymbol{\Delta}_2\|_2$.
    Otherwise, $J$ is {\em anisotropic}. 
\end{definition}



Here are three scenarios when ICD can occur.

\textbf{Scenario 1: Worst-case damage and features are orthogonal to $\bd$.}
When the non-target features are orthogonal to $\bd$, i.e., the collateral weighting is isotropic within the subspace orthogonal to the target direction $\mathbf{d}$,
this is 
equivalent to choosing $\mathbf{\Sigma} = \Pi_{\bd^\perp} := (\mathbf{I} - \mathbf{d}\mathbf{d}^T)$.

This situation is mathematically equivalent to minimizing the {\em worst-case collateral damage}, where the objective is the supremum of the change over all unit features $\mathbf{f}$ orthogonal to $\mathbf{d}$:
\begin{multline*}
\min_{\mathbf{x} \in \mathcal{M}} \sup_{\mathbf{f} : \|\mathbf{f}\|=1, \mathbf{d}^T\mathbf{f}=0} \left\{ \mathbf{f}^T(\mathbf{x}-\mathbf{h}) \right\}^2 \\
\iff \min_{\mathbf{x} \in \mathcal{M}} \|\Pi_{\bd^\perp} (\mathbf{x}-\mathbf{h}) \|^2.
\end{multline*}
This problem admits a unique, closed-form solution. The optimal point $\mathbf{x}^*$ is the one on the manifold $\mathcal{M}$ that lies within the two-dimensional plane spanned by $\mathbf{h}$ and $\mathbf{d}$:
\begin{equation}
\label{eq:slerp}
\mathbf{x}^* = \alpha\mathbf{d} + r \frac{\mathbf{h} - (\mathbf{d}^T\mathbf{h})\mathbf{d}}{\|\mathbf{h} - (\mathbf{d}^T\mathbf{h})\mathbf{d}\|},
\end{equation}
where $r = \sqrt{1-\alpha^2}$. This solution is a Slerp from $\mathbf{h}$ toward $\mathbf{d}$, stopped at the point where the alignment budget $\mathbf{d}^T\mathbf{x} = \alpha$ is met.

\begin{figure*}[!t]
    \centering
    
    \begin{subfigure}{0.48\textwidth}
        \centering
        \includegraphics[width=\linewidth]{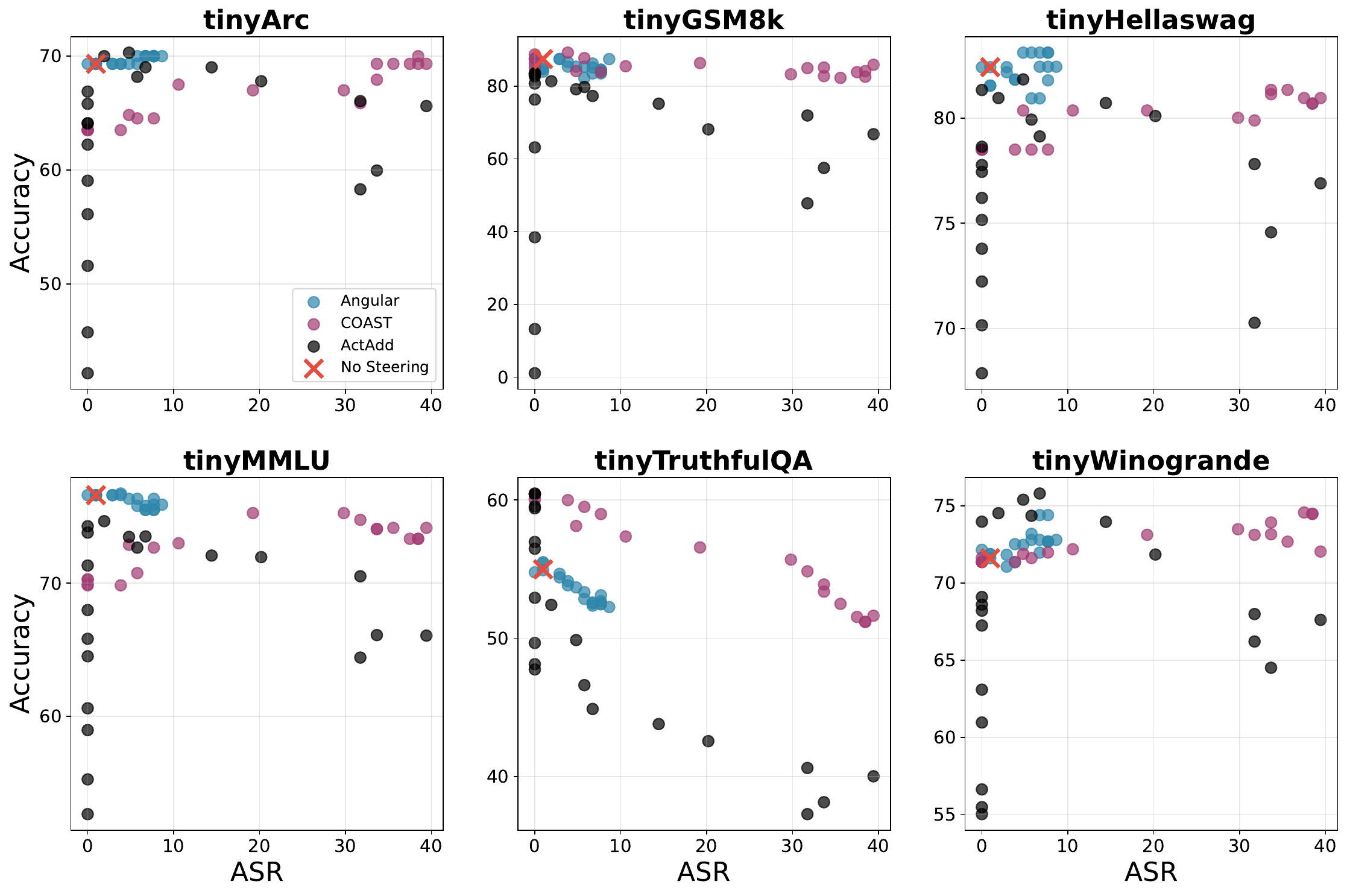}
        \caption{Gemma-2-9b-it}
    \end{subfigure}
    \hfill
    \begin{subfigure}{0.48\textwidth}
        \centering
        \includegraphics[width=\linewidth]{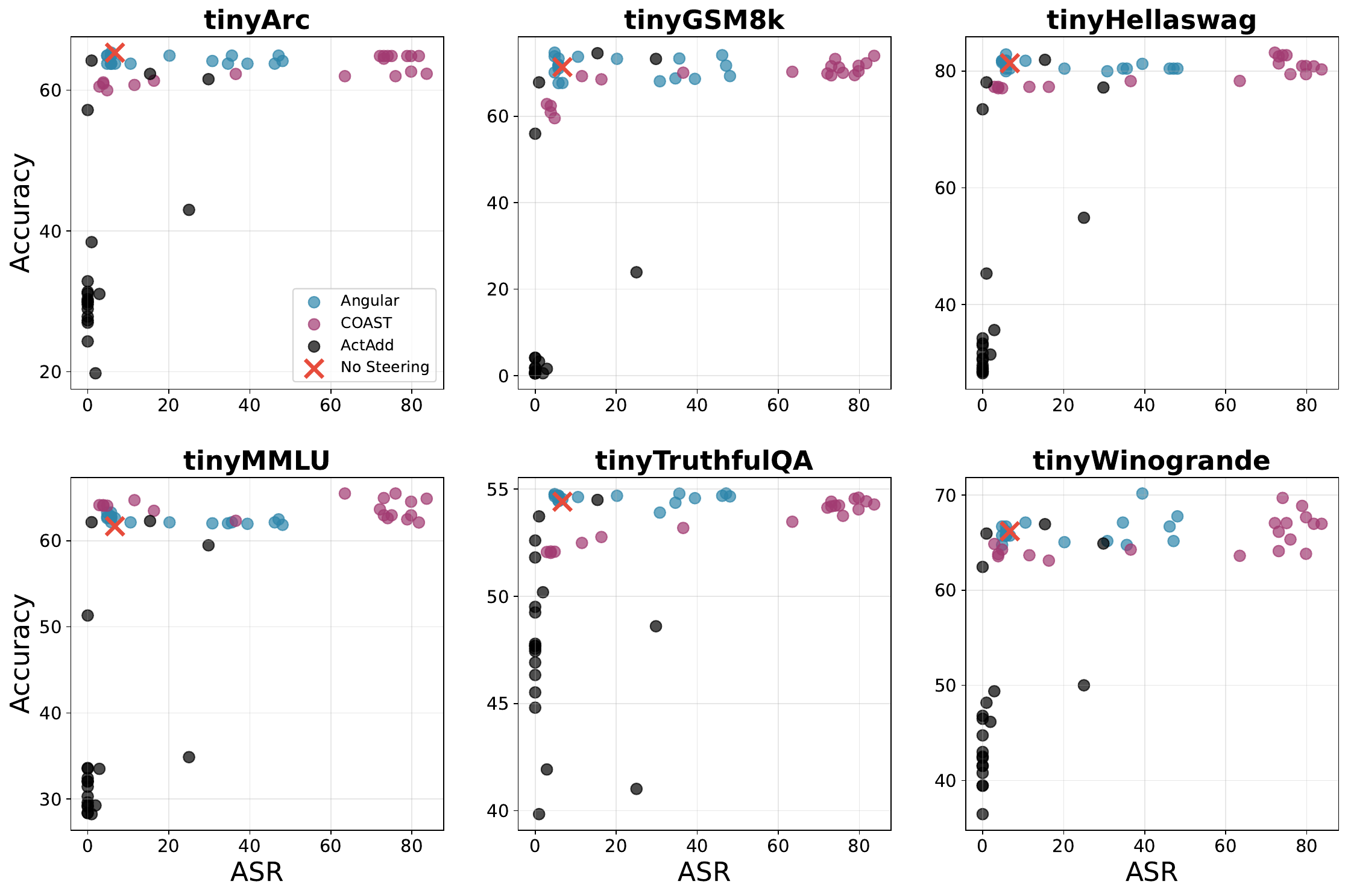}
        \caption{Llama-3.1-8B-Instruct}
    \end{subfigure}

    \vspace{1em} 

    \begin{subfigure}{0.48\textwidth}
        \centering
        \includegraphics[width=\linewidth]{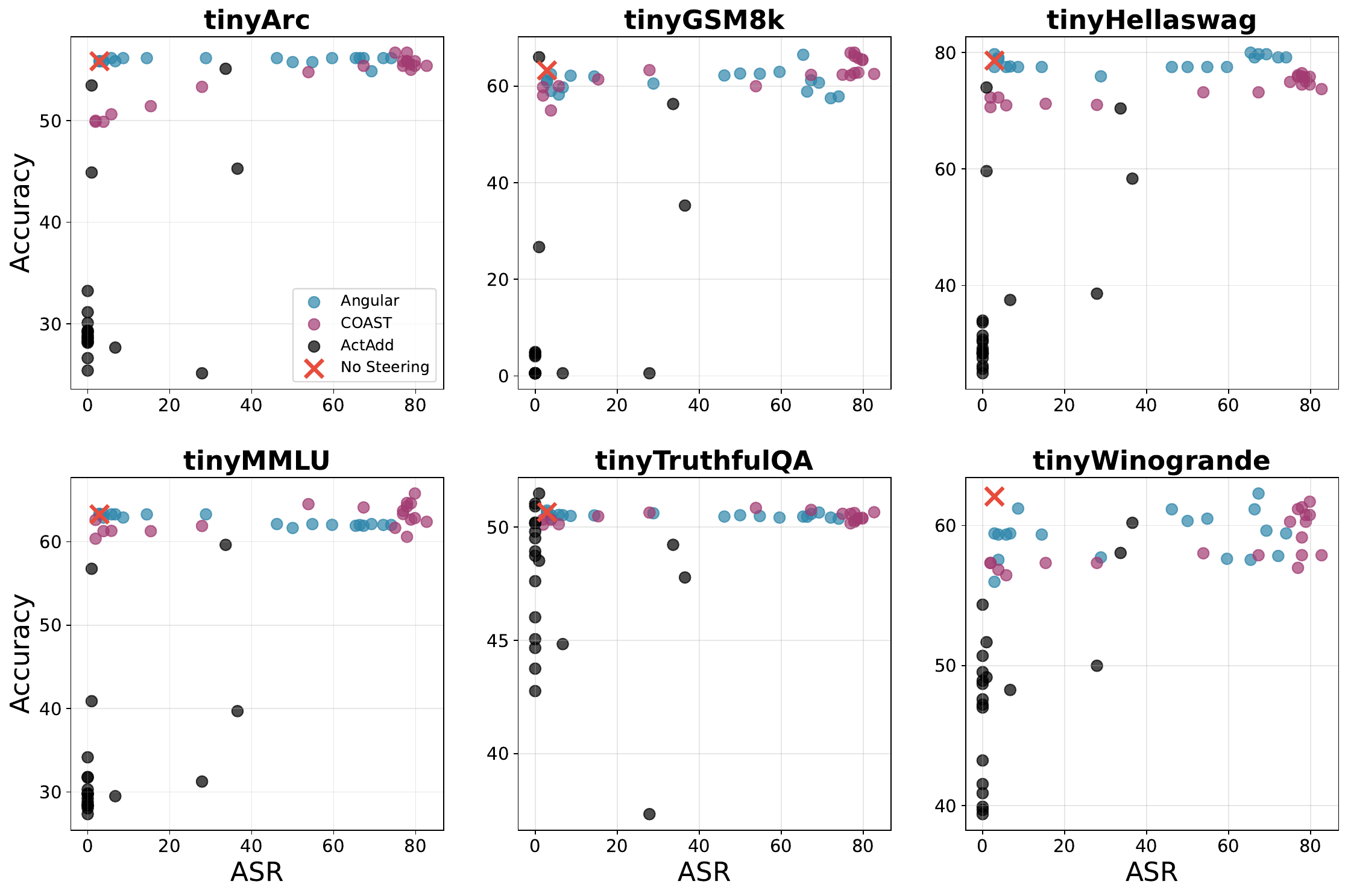}
        \caption{Llama-3.2-3B-Instruct}
    \end{subfigure}
    \hfill
    \begin{subfigure}{0.48\textwidth}
        \centering
        \includegraphics[width=\linewidth]{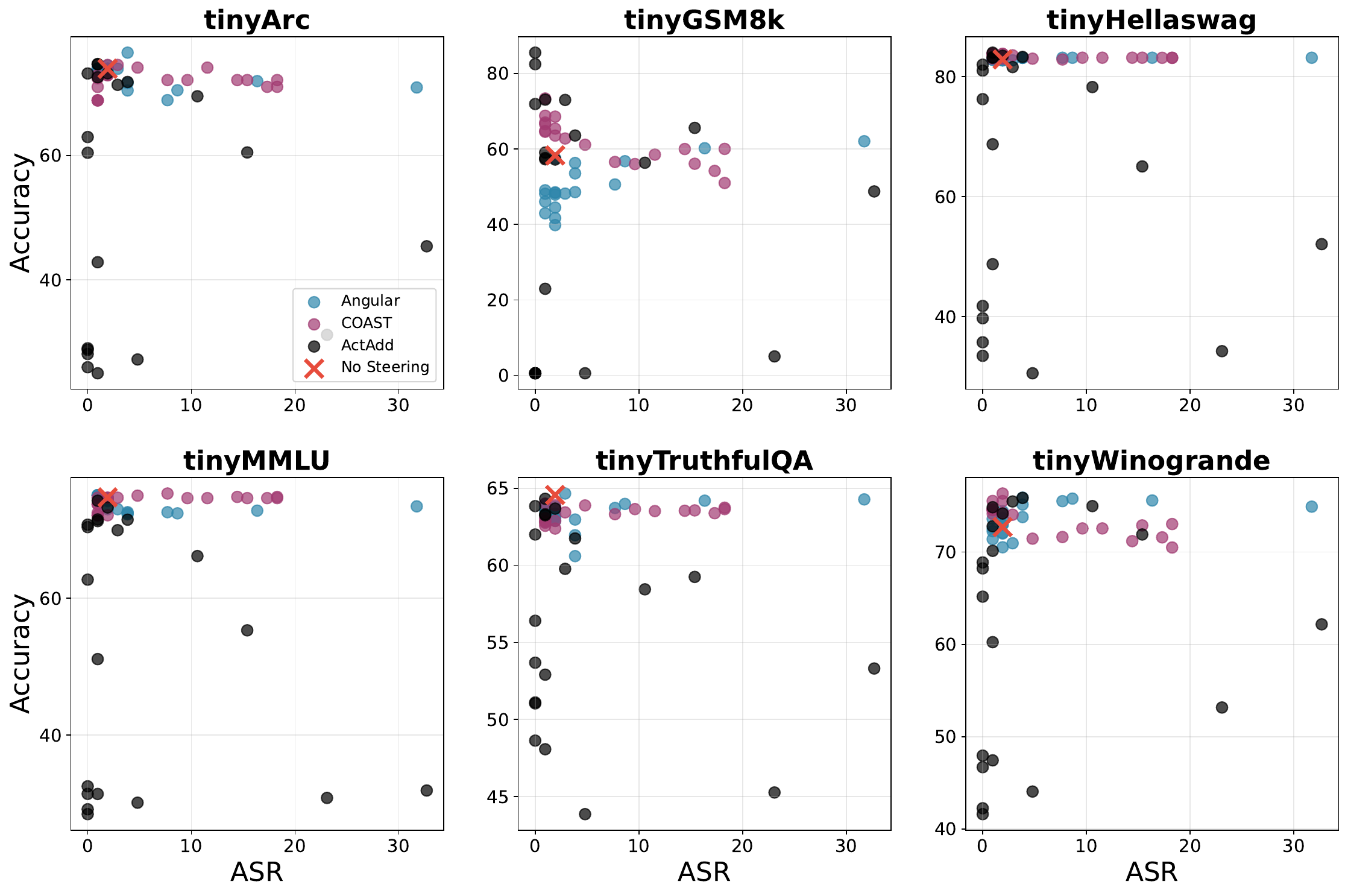}
        \caption{Qwen2.5-14B-Instruct}
    \end{subfigure}

    \caption{\small \textbf{Trade-off Analysis:} Accuracy (\%) vs.\ Attack Success Rate (ASR) on tinyBenchmarks across four models. COAST (ours) consistently maintains higher task accuracy while driving higher attack success rates compared to the Angular Steering, ActAdd and No Steering baselines.}
    \label{fig:model-comparison-pareto}
\end{figure*}

\textbf{Scenario 2: Worst-case damage and features are nearly orthogonal to $\mathbf{d}$.}
Motivated by the superposition hypothesis~\cite{elhage2022superposition}, one can relax the strict orthogonality condition $\mathbf{d}^\top \mathbf{f}=0$ to allow a small overlap:
\(
|\mathbf{d}^\top \mathbf{f}| \le \varepsilon
\)
for some small positive $\varepsilon$. Consider the corresponding worst-case collateral objective
\begin{equation}
\label{eq:near-orthogonal-obj}
\min_{\mathbf{x}\in\mathcal{M}}\ \sup_{\mathbf{f}:\ \|\mathbf{f}\|=1,\ |\mathbf{d}^\top \mathbf{f}|\le\varepsilon}
\left\{\mathbf{f}^\top(\mathbf{x}-\mathbf{h})\right\}^2 .
\end{equation}
On the budget sphere $\mathcal{M}$, the change along the target direction $\bd$ is fixed
\(
\mathbf{d}^\top(\mathbf{x}-\mathbf{h})=\alpha-\mathbf{d}^\top\mathbf{h}.
\)
As a result, the supremum in~\eqref{eq:near-orthogonal-obj} depends on $\mathbf{x}$ only through the change in the orthogonal complement of $\bd$, which is
\(
\|\Pi_{\mathbf{d}^\perp}(\mathbf{x}-\mathbf{h})\|
\). Hence, the minimizer is identical to the orthogonal case, and the optimal solution remains the same closed-form Slerp point in $\mathrm{span}\{\mathbf{h},\mathbf{d}\}$ given in~\eqref{eq:slerp}.
The full derivation is given in Appendix~\ref{subsec:near-orthogonal}.

\textbf{Scenario 3: Uniform-on-sphere features.}
If the collateral empirical features are distributed uniformly on the sphere $S^{p-1}$, then the collateral weighting matrix
is proportional to the identity, $\mathbf{\Sigma} \propto \mathbf{I}$. The objective then simplifies to minimizing the squared Euclidean distance, $J(\mathbf{x}) \propto \|\mathbf{x} - \mathbf{h}\|^2$.

Similarly, minimizing this objective on the budget slice $\mathbf{d}^T\mathbf{x} = \alpha$ is achieved by the same one-shot geodesic: a rotation of $\mathbf{h}$ toward $\mathbf{d}$ along the great circle until the constraint is met. The full derivation is given in Appendix~\ref{subsec:uniform-sphere}.



\section{Theoretical analysis of COAST}

This section establishes optimization guarantees for COAST in the general anisotropic collateral damage.
Although the collateral damage objective $J(\mathbf{x})=(\mathbf{x}-\mathbf{h})^\top \mathbf{\Sigma}(\mathbf{x}-\mathbf{h})$ is quadratic in the ambient space, the feasible set
$\mathcal{M}$ is a curved manifold, and so closed-form solutions are unavailable beyond the three ICD scenarios discussed above.
We first show that the continuous-time Riemannian gradient flow converges to a single stationary point, and then show that with the geometric initialization used in Algorithm~\ref{alg:geodesic_descent}, the flow converges to the global optimum.

\begin{lemma}[Convergence of gradient flow on the spherical slice]
\label{lem:convergence}
Let $\mathcal{M} = \{ \mathbf{x} \in \mathbb{R}^p \mid \|\mathbf{x}\| = 1,\ \mathbf{d}^\top\mathbf{x} = \alpha \}$, where $\|\mathbf{d}\|=1$ and $|\alpha| < 1$.
Consider $J(\mathbf{x}) = (\mathbf{x} - \mathbf{h})^\top \mathbf{\Sigma} (\mathbf{x} - \mathbf{h})$, where $\mathbf{\Sigma}$ is symmetric PSD.
Let $\mathbf{x}(t)$ be the solution to the Riemannian gradient flow on $\mathcal{M}$,
\[
\dot{\mathbf{x}}(t) = -\mathrm{grad}\, J(\mathbf{x}(t)), \qquad \mathbf{x}(0)\in\mathcal{M}.
\]
Then $\mathbf{x}(t)$ converges to a single critical point $\mathbf{x}^*$ as $t\to\infty$.
\end{lemma}

In the three ICD cases, the COAST objective reduces to a metric proportional to $\mathbf{I}$ or $\Pi_{\mathbf{d}^\perp}$ and admits a one-shot SLERP solution.
In the general anisotropic case, the manifold constraints prevent a closed-form update, and we instead optimize using Riemannian gradient descent.
The next theorem shows that the continuous-time flow converges to the global optimum when initialized at the Euclidean projection of $\mathbf{h}$ onto the budget manifold.
The full proof of Lemma~\ref{lem:convergence} is provided in Appendix~\ref{subsec:gradient-descent}.

\begin{theorem}[Global convergence to optimal steering]
\label{thm:global}
Let $\mathbf{\Sigma}\in\mathbb{R}^{p\times p}$ be symmetric PSD, and consider the same optimization problem in Lem.~\ref{lem:convergence} $\min_{\mathbf{x} \in \mathcal{M}}J(\bx)$.
Let $\mathbf{x}(t)$ solve the Riemannian gradient flow $\dot{\mathbf{x}}(t) = -\mathrm{grad}\,J(\mathbf{x}(t))$, and $\mathbf{x}(0)=\mathbf{x}_0$ initialized at the Euclidean projection of $\mathbf{h}$ onto $\mathcal{M}$,
\[
\mathbf{x}_0 = \alpha \mathbf{d} + r\, \frac{\mathbf{h} -  (\mathbf{d}^\top\mathbf{h})\mathbf{d}} {\|\mathbf{h} - (\mathbf{d}^\top\mathbf{h})\mathbf{d}\|}, \qquad r=\sqrt{1-\alpha^2}. \]
Then $\mathbf{x}(t)$ converges to the global minimizer $\mathbf{x}^*$ of $J$ on $\mathcal{M}$ as $t\to\infty$.
\end{theorem}

The proof of Theorem~\ref{thm:global} is provided in Appendix~\ref{subsec:proof-global}.

\textbf{Step-size selection.}
Since Algorithm~\ref{alg:geodesic_descent} discretizes the continuous flow of Theorem~\ref{thm:global} via Riemannian gradient descent, we next derive step-size bounds for the exponential map that guarantee objective decrease and solver stability.
\begin{prop}[Choice of step-size]
\label{prop:eta}
    Consider the same objective and manifold constraints in Thm.~\ref{thm:global}. Let $\Pi_{\bd^{\perp}} := I - \bd\bd^\top$ denote the orthogonal projector onto the subspace orthogonal to $\bd$. Following Alg.~\ref{alg:geodesic_descent}, the Riemannian gradient descent update given by $\bx_{t+1} = \exp_{\bx_t}(\eta \mathbf{\bxi}_t)$ (as in~\eqref{eq:exp}), where $\mathbf{\bxi}_t = -\mathrm{grad}J(\bx_t)$,
    where $\mathrm{grad}J$ denotes the Riemannian gradient on $\mathcal{M}$. The function $J$ is geodesically $L$-smooth with Lipschitz constant $L = 2\|\Pi_{\bd^{\perp}} \mathbf{\Sigma} \Pi_{\bd^{\perp}}\|_2$.
    Consequently, the iterates satisfy the descent inequality $J(\bx_{t+1}) \leq J(\bx_t) - \frac{\eta}{2} \|\mathrm{grad}J(\bx_t)\|^2$, for any step size satisfying 
    \begin{equation}
        0 < \eta \leq \frac{1}{L}
    \end{equation}
\end{prop}
The proof of Proposition~\ref{prop:eta} is given in Appendix~\ref{subsec:proof-eta}.

Proposition~\ref{prop:eta} implies that, with sufficiently small step-size, COAST makes consistent progress whenever the gradient is non-zero. Accumulating this progress over a finite horizon $T$, we derive a sublinear convergence rate to a stationary point (see Appendix~\ref{app:sublinear_rate}).
\begin{table*}[t]
\centering
\caption{\small Comparison of COAST and Slerp. We report ASR (Alignment Success Rate, $\uparrow$), PPL (Perplexity, $\downarrow$), and Acc (Task Accuracy, $\uparrow$) averaged across different steering intensities $\theta \in [0^\circ, 180^\circ]$ (see Appendix~\ref{app:steering_strength} regarding steering intensities). Individual task performance is measured in accuracy.
}
\label{tab:main_results}
\resizebox{\textwidth}{!}{%
\begin{tabular}{llccccccccc}
\toprule
\textbf{Model} & \textbf{Method} & \textbf{Avg ASR (\%)} $\uparrow$ & \textbf{Avg PPL} $\downarrow$ & \textbf{Avg Acc} $\uparrow$ & \textbf{tinyArc} & \textbf{tinyHellaswag} & \textbf{tinyMMLU} & \textbf{tinyTruthfulQA} & \textbf{tinyWinogrande} & \textbf{tinyGSM8k} \\
\midrule
{Llama-3.1-8B-Instruct} & Slerp & 51.67 & 6.60 & \textbf{65.95} & \textbf{63.41} & \textbf{80.08} & \textbf{63.90} & \textbf{54.01} & \textbf{66.63} & 67.65 \\
 & COAST & \textbf{52.13} & \textbf{6.53} & 65.70 & 62.86 & 79.68 & 63.78 & 53.53 & 65.52 & \textbf{68.80} \\
\midrule
{Llama-3.2-3B-Instruct} & Slerp & 52.68 & 5.67 & 60.10 & 54.07 & \textbf{75.00} & 62.33 & \textbf{50.49} & 58.41 & 60.29 \\
 & COAST & \textbf{54.76} & \textbf{5.56} & \textbf{60.43} & \textbf{54.15} & 73.85 & \textbf{62.83} & 50.44 & \textbf{58.76} & \textbf{62.57} \\
\midrule
{Qwen2.5-14B-Instruct} & Slerp & \textbf{7.64} & 2.81 & \textbf{71.41} & \textbf{72.47} & 83.28 & \textbf{74.24} & \textbf{63.46} & 73.29 & 61.74 \\
 & COAST & 6.93 & \textbf{2.80} & 71.39 & 72.06 & \textbf{83.34} & 74.11 & 63.31 & \textbf{73.45} & \textbf{62.03} \\
\midrule
{Gemma-2-9B-it} & Slerp & 19.28 & {2.15} & 71.88 & 66.45 & 79.85 & 72.66 & 56.16 & \textbf{72.74} & 83.39 \\
 & COAST & \textbf{19.48} & \textbf{2.15} & \textbf{72.17} & \textbf{66.48} & \textbf{79.87} & \textbf{72.68} & \textbf{56.18} & 72.62 & \textbf{85.21} \\
\bottomrule
\end{tabular}%
}
\end{table*}

\section{Experiments} \label{sec:bottleneck}

In this section, we demonstrate the effectiveness of COAST across a diverse set of instruction-tuned LLMs spanning various families and scales: Llama-3.2-3B-Instruct~\cite{dubey2024llama}, Llama-3.1-8B-Instruct~\cite{dubey2024llama}, Qwen2.5-14B-Instruct~\cite{qvq-72b-preview}, and Gemma-2-9B-IT~\cite{team2024gemma}. We compare against four baselines: ActAdd, Angular Steering, Slerp, and No Steering.

To assess steering effectiveness, we define the {\em Attack Success Rate} (ASR) as the fraction of harmful instructions for which the steered model successfully produces a jailbreak completion (higher indicates stronger steering). 
We compute the ASR using HarmBench, a standardized evaluation framework for automated red teaming and robust refusal.

To evaluate performance preservation, we report task accuracy on the tinyBenchmarks suite of tasks~\cite{polo2024tinybenchmarks}, including 100 examples of ARC~\cite{clark2018think}, MMLU~\cite{hendrycks2021mmlu}, Winogrande~\cite{sakaguchi2021winogrande}, GSM8K~\cite{cobbe2021gsm8k}, TruthfulQA~\cite{lin2022truthfulqa}, and HellaSwag~\cite{zellers-etal-2019-hellaswag}.

Appendix~\ref{app:exp_details} contains the experimental details: (i)~intervention locations (Appendix~\ref{app:intervention_locations}), (ii) target direction construction (Appendix~\ref{app:steering_vector_construction}), (iii) collateral matrix estimation and optimization hyperparameters (Appendix~\ref{app:sigma_compute}), (iv) steering strength (Appendix~\ref{app:steering_strength}), (v) adaptive alignment budgets (Appendix~\ref{app:aab}) (vi) norm-preserving implementation (Appendix~\ref{app:norm_preserve}).



\textbf{Trade-off between jailbreak success and general capability.}
We plot ASR against Task Accuracy to evaluate the ability of steering methods to compel harmful compliance while minimizing capability degradation. As shown in Fig.~\ref{fig:model-comparison-pareto}, ActAdd exhibits a steep trade-off—sacrificing accuracy to achieve ASR—while Angular Steering struggles to yield high ASR. In contrast, COAST consistently drives high ASR while maintaining task performance comparable to unsteered models. Effectively, COAST achieves robust steering with minimal overhead, delivering up to a 30\% improvement in ASR over Angular Steering and a 20\% accuracy advantage over ActAdd at equivalent ASR levels.






\textbf{Comparison with Slerp.}
 We compare COAST to Slerp by averaging metrics across steering intensities $\theta \in [0^\circ, 180^\circ]$ (see App.~\ref{app:steering_strength}). Table~\ref{tab:main_results} demonstrates that COAST achieves a higher Alignment Success Rate (ASR) across three of the four evaluated models. For Gemma-2-9B and Llama-3.2-3B specifically, COAST improves both ASR and task accuracy simultaneously. In cases where Slerp retains slightly higher accuracy, the margin is small, suggesting that COAST increases steering efficacy without capability degradation.

 Theoretically, as a generalization of Slerp, COAST should perform at least as well as the baseline. We hypothesize that instances where Slerp slightly outperforms COAST are attributable to optimization dynamics. In regions where the collateral damage is nearly isotropic—where Slerp provides the exact closed-form solution—COAST approximates this solution via gradient descent. Consequently, with aggressive step sizes, the optimization may occasionally deviate from the optimal path, a limitation resolvable with finer hyperparameter tuning.

\textbf{Computational Efficiency.}
We evaluated efficiency by measuring inference throughput in tokens/second. COAST incurs a marginal overhead, slower than the unsteered baseline by 0.9\%--4.0\% and than the Slerp baseline by less than 1.5\%. Detailed results are provided in Appendix~\ref{subsec:computational-cost} and Table~\ref{tab:throughput}.

\textbf{Analysis: Collateral damage as proxy measure of performance degradation.}
To validate collateral damage as a proxy for performance degradation, we evaluated ActAdd and COAST on Qwen2.5-14B-Instruct with the tinyBenchmarks. For each task, we computed the collateral damage—defined as $(\bx - \bh)^T \bSigma (\bx - \bh)$—averaged over all layers and activations. Each point in Fig.~\ref{fig:damage_correlation} represents a steering method applied with different steering coefficients (see Appendix~\ref{app:steering_strength}), plotted against its corresponding downstream accuracy. We observe a strong negative Pearson correlation ($r < -0.9$) across all tasks, confirming that our metric is a reliable predictor of model performance decline.

\section{Conclusions}

COAST turns activation steering from an ad hoc editing trick into a {\em principled, geometry-aware control method}: for a chosen steering direction, it finds the least disruptive norm-preserving intervention that still achieves a specified alignment budget. By explicitly modeling anisotropy through the empirical second-moment structure, COAST replaces the hidden isotropy assumption behind many steering operators with a metric that reflects which directions in activation space are actually costly to perturb (causing collateral damage). This yields a practical steering regime that improves the effectiveness–capability tradeoff, enabling targeted behavior changes without sacrificing unrelated competencies. More broadly, COAST provides a foundation for accountable activation interventions (steering with explicit objectives, transparent constraints, and predictable side effects) that bring behavior control closer to something we can deploy with confidence.

\section*{Impact Statement}

This paper contributes to safer and more reliable use of large language models by reducing unintended side effects of activation steering. By explicitly optimizing for minimal collateral change under a fixed steering budget, COAST helps prevent capability degradation and unexpected behavioral drift when deploying steering-based controls for alignment, personalization, or content moderation. Our method can lower the risk of overbroad interventions that suppress helpful behaviors or amplify spurious ones, supporting more predictable model behavior in real-world systems. At the same time, steering remains a powerful intervention tool that could be misused to elicit harmful or deceptive behaviors; we therefore emphasize evaluation across diverse tasks, transparent reporting of side effects, and responsible access controls for deployment.

\section*{Acknowledgments}
This work was supported by ONR grant N00014-23-1-2714, DOE grant DE-SC0020345, DOI grant 140D0423C0076, and a Google Cloud Computing Award.

\bibliographystyle{icml2026}
\bibliography{references}

\clearpage
\appendix
\onecolumn
\label{sec:appendix}

{\small
\startcontents[appendixcontents]
  \printcontents[appendixcontents]{ }{1}{\setcounter{tocdepth}{2}\textbf{Table of Contents}\vskip3pt\hrule\vskip5pt}
  \vskip3pt\hrule\vskip5pt
}

\section{Technical Proofs}
\label{sec:technical-proofs}

\subsection{Proof of Lem.~\ref{lem:convergence} (Convergence of Gradient Flow on Spherical Slice)}
\label{subsec:gradient-descent}

$J(\bx) = (\bx - \bh)^\top \mathbf{\Sigma} (\bx - \bh)$ over $\mathcal{M} := \{ \bx : \|\bx\| = 1, \bd^\top \bx = \alpha \}$.

Consider any $\bx \in \mathcal{M}$ and any orthogonal basis $\{ \be_1, \be_2, \dots, \be_{p-2} \}$ of $T_{\bx}\mathcal{M}$.

Define $\text{grad } J(\bx) = \sum_{i=1}^{p-2} dJ_{\bx}(\be_i)\be_i \in T_{\bx}\mathcal{M}$.

Let $\bx(t)$ solve the gradient flow ODE on $\mathcal{M}$:
$$ \dot{\bx}(t) = - \text{grad } J(\bx(t)) $$
Then $\bx(t)$ converges to a single point as $t \to \infty$.

\textbf{Proof:}

\textbf{Claim 1:} $J$ is non-increasing.

\textit{Proof:}
$$ \frac{d}{dt} J(\bx(t)) = dJ_{\bx(t)}(\dot{\bx}(t)) = dJ_{\bx(t)}\left( -\sum_{i=1}^{p-2} dJ_{\bx}(\be_i)\be_i \right) $$
$$ = \sum_{i=1}^{p-2} -dJ_{\bx}(\be_i)^2 \langle \be_i, \be_i \rangle_g \quad (\text{since } \|\be_i\| = 1) $$
$$ = \langle \text{grad } J_{\bx}(t), -\text{grad } J_{\bx}(t) \rangle = -\|\text{grad } J_{\bx}(t)\|^2 \leq 0 $$

Since $\mathcal{M}$ is a $p-2$ sphere, it is closed and bounded. $\mathcal{M}$ is compact.
Hence any infinite sequence in $\mathcal{M}$ must have at least one accumulation point in $\mathcal{M}$. Since $\bx(t) \in \mathcal{M}$ for all $t$, the sequence $\bx(t)$ has at least one accumulation point.

Let $\Omega$ be the set containing all accumulation points of $\bx(t)$.

\textbf{Claim 2:} $\Omega$ is non-empty, compact, connected.

\textit{Proof:}
\begin{itemize}
    \item \textbf{Non-empty:} (above discussion, because $\mathcal{M}$ is compact).
    \item \textbf{Compact:} Let $\{y_n\}$ be a convergent sequence in $\Omega$ (if there is any). Since $y_1 \in \Omega$, we can find $t_1$ s.t. $\bx(t_1)$ is arbitrarily close to $y_1$. For $y_2$, we can find $t_2 > t_1$ s.t. $\bx(t_2)$ is arbitrarily close to $y_2$, and so on.
    Hence the sequence $\bx(t_1), \bx(t_2), \dots$ gets infinitely close to $\{y_1, y_2, \dots\}$. But $\{y_n\}$ converges to $y$. Meanwhile $\{\bx(t_1), \bx(t_2), \dots\}$ is a subsequence of $\bx(t)$, hence its limit is in $\Omega$. Therefore $\Omega$ is closed. And a closed subset of a compact set is also compact.
    \item \textbf{Connected:} Suppose $\Omega$ is disconnected. WLOG, assume $\Omega = \Omega_1 \cup \Omega_2$, $\Omega_1 \cap \Omega_2 = \emptyset$. Since $p_1 \in \Omega_1, p_2 \in \Omega_2$ are accumulation points of $\bx(t)$, $\bx(t)$ goes arbitrarily near $p_1, p_2$ infinite times. Hence $\bx(t)$ goes through the space between $\Omega_1, \Omega_2$ infinite times.
    Since that middle set is also compact, a limit of the infinite subsequence of $\bx(t)$ in that set also has an accumulation point in that set. $\bx(t)$ has an accumulation point in that set, which also must be in $\Omega$. So by contradiction, $\Omega$ must be connected.
\end{itemize}

Note that $\mathcal{M}$ is compact and positive invariant w.r.t $\bx(t)$ (Because the entire trajectory stays in $\mathcal{M}$ from start $t_0 \to \infty$).
Also notice that $\frac{d}{dt}J(\bx(t)) \leq 0 \quad \forall t$. Let the set of critical points be:
$$ E = \{ \bx(t) \in \mathcal{M} \mid \text{grad } J(\bx(t)) = 0 \} $$
Let $\mathcal{I}$ be the largest Positive invariant set in $E$ such that all trajectories starting in $\mathcal{I}$ will stay in $\mathcal{I}$ forever.
By LaSalle's Invariance Principle, $\bx(t)$ (as starting in $\mathcal{M}$) will approach $\mathcal{I}$ as $t \to \infty$. Hence the accumulation-point set $\Omega \subseteq \mathcal{I} \subseteq E$.

\textbf{Claim 3:} $E$ is finite.

\textit{Proof:}
The problem becomes:
$$ \min \quad (\bx - \bh)^\top \bSigma (\bx - \bh) \quad \text{s.t.} \quad \|\bx\|=1, \quad \bd^\top \bx = \alpha $$
We reason with:
$$ \bx = \alpha \bd + r \bu, \quad r = \sqrt{1-\alpha^2}, \quad \text{where } \bu \in \bd^{\perp} $$
$$ \phi(\alpha \bd + r \bu) = (\alpha \bd + r \bu - \bh)^\top \mathbf{\Sigma} (\alpha \bd + r \bu - \bh) $$
$$ = r^2 \bu^\top \mathbf{\Sigma} \bu + 2(\alpha \bd - \bh)^\top \mathbf{\Sigma} r \bu + (\alpha \bd - \bh)^\top \mathbf{\Sigma} (\alpha \bd - \bh) $$
Notice that since $\bu \in \bd^{\perp}$, $\bu = \Pi_{\bd^\perp} \bu$ where $\Pi_{\bd^\perp} = \bI - \bd \bd^\top$ (the orthogonal projection matrix on the ortho complement of $\bd$).
We can rewrite the problem as:
$$ \min \quad \phi(\bu) \quad \text{s.t.} \quad \|\bu\|=1, \quad \bu^\top \bd = 0 $$
$$ \phi(\bu) = r^2 \bu^\top \Pi_{\bd^\perp} \mathbf{\Sigma} \Pi_{\bd^\perp} \bu + 2r(\alpha \bd - \bh)^\top \mathbf{\Sigma} \Pi_{\bd^\perp} \bu $$
$$ = r^2 \bu^\top \bA \bu + 2r \bb^\top \bu $$
(Note: We bake the constraint $\bu^\top \bd = 0$ into $\bA$ and $\bb$).

Use Lagrange multiplier to solve optimality for $\phi(\bu)$:
The first order condition:
$$ \mathcal{L}(\bu, \lambda, \mu) = \phi(\bu) + \frac{1}{2} \lambda (\|\bu\|^2 - 1) + \mu (\bu^\top \bd) $$
$$ \nabla_{\bu} \mathcal{L} = \nabla_{\bu} \phi(\bu) + \lambda \bu + \mu \bd = 0 $$
Consider $\nabla_{\bu} \phi(\bu) = 2r^2 \Pi_{\bd^\perp} \mathbf{\Sigma} \Pi_{\bd^\perp} \bu + 2r \Pi_{\bd^\perp} \mathbf{\Sigma} (\alpha \bd - \bh)$.
Since $\Pi_{\bd^\perp}^\top \bd = \Pi_{\bd^\perp} \bd = 0$, multiplying by $\bd^\top$ implies $\mu = 0$.
Hence the Lagrangian is simplified to:
$$ \mathcal{L}(\bu, \lambda) = \phi(\bu) + \lambda (1 - \|\bu\|) \quad \text{(simplified form)} $$
The first order condition becomes:
$$ 2r^2 \bA \bu + 2r \bb = 2 \lambda \bu $$
\begin{equation}
\label{eq:lambda-system}
\Leftrightarrow (r^2 \bA - \lambda \bI) \bu = -r \bb
\end{equation}

We analyze the solution of~\eqref{eq:lambda-system} via $\lambda$. First, let's diagonalize $\bA = \bQ \bD \bQ^\top$ so that we can solve orthogonal component-wise. Notice that $\bA$ is diagonalizable because it is a symmetric matrix.
Let $\by = \bQ^\top \bu$ and $\bc = \bQ^\top \bb$.
Equation~\eqref{eq:lambda-system} becomes:
$$ r^2 \bD \by - \lambda \by = -r \bc $$
$$ r^2 \beta_i y_i - \lambda y_i = -r c_i, \quad i=1, \dots, p-2 $$
$$ y_i = \frac{-r c_i}{r^2 \beta_i - \lambda} $$
We enforce the constraint:
$$ \|\by\|^2 = \by^\top \by = \bu^\top \bQ \bQ^\top \bu = \bu^\top \bu = 1 $$
$$ \sum_{i=1}^{p-2} y_i^2 = \sum_{i=1}^{p-2} \frac{r^2 c_i^2}{(r^2 \beta_i - \lambda)^2} = 1 $$
\begin{equation}
\label{eq:polynomial}
\Rightarrow \left( \sum_{i} r^2 c_i^2 \prod_{j \neq i} (r^2 \beta_j - \lambda)^2 \right) - \prod_{k}^{p-2} (r^2 \beta_k - \lambda)^2 = 0
\end{equation}
Equation~\eqref{eq:polynomial} is a polynomial of $\lambda$ with degree $2(p-2)$. It has finite number of roots (at most $2(p-2)$).
Any valid critical point of $\phi(\bx)$ must be generated by $\lambda$ that solves~\eqref{eq:polynomial}. So does any critical point of $J(\bx)$.
Therefore the set of critical points of $J$ is Finite.
$\Rightarrow \Omega$, the set of accumulation points of $\bx(t)$ is also finite.
Since $\Omega$ is also connected, it must have only 1 element.
Hence, $\bx(t)$ converges to a critical point.

\subsection{Proof of Lem.~\ref{lem:exp-map} (Exponential Map on $\mathcal{M}$)}
\label{subsec:proof-thm1}
\begin{proof}
Define the center and radius
\[
\bc := \alpha \bd, \qquad r := \sqrt{1 - \alpha^2}.
\]
For any $\bx \in \mathcal{M}$ we have $\bd^\top \bx = \alpha$ and $\|\bx\|=1$, hence
\[
\bx = \alpha \bd + (\bx - \alpha \bd), \qquad \bd^\top (\bx - \alpha \bd) = 0,
\]
and
\[
\|\bx - \alpha \bd\|^2
= \|\bx\|^2 - \alpha^2
= 1 - \alpha^2
= r^2.
\]
Thus every $\bx \in \mathcal{M}$ can be written as
\[
\bx = \bc + r \bu, \qquad \text{with } \|\bu\| = 1,\ \bu \perp \bd.
\]
Therefore $\mathcal{M}$ is a $(p-2)$-dimensional sphere of radius $r$ centered at $\bc$ in the hyperplane $\bd^\perp$.

Consider the map
\[
\Phi : \mathcal{M} \to \mathbb{R}^p, \qquad \Phi(\bx) := \bu := \frac{\bx - \bc}{r}.
\]
For $\bx \in \mathcal{M}$, we have $\|\bu\| = 1$ and $\bu \perp \bd$, so $\Phi$ maps $\mathcal{M}$ diffeomorphically onto
\[
\widetilde{\mathcal{M}} := \{\bu \in \mathbb{R}^p : \|\bu\| = 1,\ \bu \perp \bd\},
\]
which is a standard unit sphere $S^{p-2}$ in $\bd^\perp$.

Let $\bx \in \mathcal{M}$ and $\bv \in T_{\bx} \mathcal{M}$. The differential of $\Phi$ at $\bx$ is
\[
d\Phi_{\bx}(\bxi) = \frac{1}{r}\,\bxi, \qquad \forall\, \bxi \in T_{\bx} \mathcal{M}.
\]
Thus the corresponding tangent vector at $\bu = \Phi(\bx)$ is
\[
\bw := d\Phi_{\bx}(\bv) = \frac{1}{r} \bv \in T_{\bu} \widetilde{\mathcal{M}}.
\]

On the unit sphere $\widetilde{\mathcal{M}}$, geodesics are great circles.
The unique geodesic $\boldsymbol{\gamma} : \mathbb{R} \to \widetilde{\mathcal{M}}$ with initial data
\[
\boldsymbol{\gamma}(0) = \bu, \qquad \boldsymbol{\gamma}'(0) = \bw
\]
is given by
\[
\boldsymbol{\gamma}(t)
= \bu \cos(\|\bw\| t) + \frac{\bw}{\|\bw\|} \sin(\|\bw\| t).
\]
By definition of the exponential map on $\widetilde{\mathcal{M}}$,
\[
\exp^{\widetilde{\mathcal{M}}}_{\bu}(\bw)
= \boldsymbol{\gamma}(1)
= \bu \cos(\|\bw\|) + \frac{\bw}{\|\bw\|} \sin(\|\bw\|).
\]

To obtain the exponential map on $\mathcal{M}$, we pull this geodesic back via $\Phi^{-1}$.
Define
\[
\bx(t) := \Phi^{-1}(\boldsymbol{\gamma}(t)) = \bc + r\,\boldsymbol{\gamma}(t).
\]
Then
\[
\begin{aligned}
\bx(t)
&= \bc + r\left[ \bu \cos(\|\bw\| t) + \frac{\bw}{\|\bw\|} \sin(\|\bw\| t)\right] \\
&= \bc + (\bx - \bc)\cos(\|\bw\| t) + r\,\frac{\bw}{\|\bw\|}\sin(\|\bw\| t).
\end{aligned}
\]
Using $\bw = \bv/r$, we have $\|\bw\| = \|\bv\|/r$ and $\bw/\|\bw\| = \bv/\|\bv\|$, so
\[
\bx(t)
= \bc + (\bx - \bc)\cos\!\Bigl(\frac{\|\bv\|}{r} t\Bigr)
  + r\,\frac{\bv}{\|\bv\|}\sin\!\Bigl(\frac{\|\bv\|}{r} t\Bigr).
\]
By construction $\bx(0) = \bx$ and $\bx'(0) = \bv$, so $\bx(t)$ is the geodesic on $\mathcal{M}$
starting at $\bx$ with initial velocity $\bv$. Therefore
\[
\exp_{\bx}(\bv) = \bx(1)
= \bc + (\bx - \bc)\cos\!\Bigl(\frac{\|\bv\|}{r}\Bigr)
  + r\,\frac{\bv}{\|\bv\|}\sin\!\Bigl(\frac{\|\bv\|}{r}\Bigr).
\]
Substituting $\bc = \alpha \bd$ and $\tau = \|\bv\|/r$ gives
\[
\exp_{\bx}(\bv)
= \alpha \bd + (\bx - \alpha \bd)\cos(\tau)
  + \frac{\bv}{\|\bv\|}\,r\sin(\tau),
\]
which is the desired formula.
\end{proof}

\subsection{Proof of Thm.~\ref{thm:global} (Global Convergence to Optimal Steering)}
\label{subsec:proof-global}

\subsubsection{Problem Transformation}
We define the transformed variables to center the problem:
\begin{equation}
    \bu = \bx - \alpha\bd, \quad r = \sqrt{1 - \alpha^2}
\end{equation}
The objective function on the manifold is defined as:
\begin{equation}
    J(\bu) = (\bu - \tilde{\bh})^\top \bSigma (\bu - \tilde{\bh})
\end{equation}
where $\tilde{\bh} = \bh - \alpha\bd$.

\subsubsection{Spectral Decomposition}
Let $(\lambda_1, \bv_1)$ denote the smallest eigenvalue and corresponding eigenvector of the Hessian $2\bSigma$, representing the ``flattest direction'' of the quadratic valley. We decompose the shifted activation vector $\tilde{\bh}$ in the eigenbasis $\{\bv_i\}$ of $2\bSigma$:
\begin{equation}
    \tilde{\bh} = \sum_{i=1}^p \beta_i \bv_i
\end{equation}

\paragraph{Case 1: $\beta_1 = 0$ (Orthogonal Data)}
If the data vector has no component along the principal eigenvector, \textbf{Lemma 3.2} of Martínez (1994) applies. $\tilde{\bh}$ is orthogonal to the eigenvector $\bv_1$, implying that a local-nonglobal minimizer does not exist. This is the trivial case where global convergence is unconditional.

\paragraph{Case 2: $\beta_1 \neq 0$ (General Case)}
This is the nontrivial scenario. We analyze the alignment of critical points.

\begin{enumerate}
    \item \textbf{Initialization Alignment:} The algorithm initializes at the Euclidean projection:
    \begin{equation}
        \bu_0 = r \frac{\tilde{\bh}}{\|\tilde{\bh}\|} \implies \bv_1^\top \bu_0 = \frac{r}{\|\tilde{\bh}\|} \beta_1
    \end{equation}
    Consequently, the initialization preserves the sign of the data vector along the $\bv_1$:
    \begin{equation}
        \text{sgn}(\bv_1^\top \bu_0) = \text{sgn}(\beta_1)
    \end{equation}

    \item \textbf{Critical Point Characterization:} According to Martínez (1994), any critical point $\bs$ must satisfy the secular equation:
    \begin{equation}
    \label{eq:secular}
        (2\bSigma + \mu \bI)\bs = -\boldsymbol{\ell} = 2\bSigma \tilde{\bh},
    \end{equation} where $\boldsymbol{\ell} := -2\bSigma\tilde{\bh}$.
    Since $2\bSigma = \sum \lambda_i \bv_i \bv_i^\top$, where $\lambda_i \geq 0$ since $\bSigma$ is PSD, assuming a solution of~\eqref{eq:secular} has the form $\bs = \sum \gamma_i \bv_i$, we match coefficients:
    \begin{align}
        (2\bSigma + \mu \bI) \sum \gamma_i \bv_i &= \sum \gamma_i (\lambda_i + \mu) \bv_i \\
        \text{RHS} = 2\bSigma \sum \beta_i \bv_i &= \sum \lambda_i \beta_i \bv_i
    \end{align}
    Solving for $\gamma_i$:
    \begin{equation}
        \gamma_i (\lambda_i + \mu) = \lambda_i \beta_i \implies \bs = \sum_{i=1}^p \frac{\lambda_i \beta_i}{\lambda_i + \mu} \bv_i
    \end{equation}

    \item \textbf{Global Minimizer ($\bs_G$):} By Martínez Theorem 3.1, the global multiplier satisfies $\mu_G \ge -\lambda_1$. Assuming the generic case $\lambda_i + \mu_G > 0$:
    \begin{equation}
        \text{sgn}(\bv_1^\top \bs_G) = \text{sgn}(\beta_1)
    \end{equation}
    Thus, $\bs_G$ lies in the same spectral hemisphere as $\tilde{\bh}$.

    Note that if $\lambda_1 = 0$, then following Martínez Lemma 3.2, there would be no local-nonglobal minimizer. Hence, we consider $\lambda_1 > 0$.

    \item \textbf{Local-nonglobal Minimizer ($\bs_L$):} By Martínez Lemma 3.3, a local nonglobal minimizer corresponds to $\mu_L \in (-\lambda_2, -\lambda_1)$. This implies $(\lambda_1 + \mu_L) < 0$, resulting in:
    \begin{equation}
        \text{sgn}(\bv_1^\top \bs_L) = -\text{sgn}(\beta_1)
    \end{equation}
\end{enumerate}

\subsubsection{Topological Separation Analysis}

We define the the trap ($S_{opp}$), the safe basin ($S_{same}$), and the separating hyperplane ($H$) between them:
\begin{align}
    S_{opp} &= \{ \bu \in \mathcal{M} \mid \text{sgn}(\bv_1^\top \bu) = -\text{sgn}(\beta_1) \} \\
    S_{same} &= \{ \bu \in \mathcal{M} \mid \text{sgn}(\bv_1^\top \bu) = \text{sgn}(\beta_1) \} \\
    H &= \{ \bu \in \mathcal{M} \mid \bv_1^\top \bu = 0 \}
\end{align}
Note that any continuous path from the initialization $\bu_0 \in S_{same}$ to the trap $\bs_L \in S_{opp}$ must intersect $H$.

\paragraph{Boundary Flow Analysis:}
We examine the gradient flow $\dot{\bu}$ on the boundary $H$. Let $\bw = -\text{sgn}(\beta_1) \bv_1$ be the normal vector pointing into $S_{opp}$. The projected gradient dynamics are:
\begin{equation}
    \dot{\bu} = -\Pi_{\bu} \nabla J(\bu) = -\left( \bI - \frac{\bu\bu^\top}{r^2} \right) \nabla J(\bu)
\end{equation}
We project the flow onto the principal direction $\bv_1$:
\begin{align}
    \bv_1^\top \dot{\bu} &= -\bv_1^\top \left( \bI - \frac{\bu\bu^\top}{r^2} \right) \nabla J(\bu) = -\left( \bv_1^\top - \frac{(\bv_1^\top \bu)\bu^\top}{r^2} \right) \nabla J(\bu)
\end{align}
For any $\bu \in H$, we have $\bv_1^\top \bu = 0$. The expression simplifies to:
\begin{equation}
    \bv_1^\top \dot{\bu} = -\bv_1^\top \nabla J(\bu)
\end{equation}
By checking the flow relative to the ``bad'' direction $\bw$, we have:
\begin{align}
    \bw^\top \dot{\bu} &= -\text{sgn}(\beta_1) (\bv_1^\top \dot{\bu}) = \text{sgn}(\beta_1) \bv_1^\top \nabla J(\bu) = \text{sgn}(\beta_1) \bv_1^\top (2\bSigma\bu - \boldsymbol{\ell}) \\
    &= \text{sgn}(\beta_1) \bv_1^\top (2\bSigma)\bu - \text{sgn}(\beta_1) \bv_1^\top \boldsymbol{\ell}
\end{align}
Analyzing the two terms:
\begin{enumerate}
    \item $\bv_1^\top (2\bSigma) = \lambda_1 \bv_1^\top \implies \bv_1^\top (2\bSigma) \bu = \lambda_1 (\bv_1^\top \bu) = 0$ (since $\bu \in H$).
    \item $\bv_1^\top \boldsymbol{\ell} = -\bv_1^\top (2\bSigma) \tilde{\bh} = -\lambda_1 \beta_1$.
\end{enumerate}
Substituting back:
\begin{equation}
    \bw^\top \dot{\bu} = 0 - \text{sgn}(\beta_1) (-\lambda_1 \beta_1) = -\lambda_1 |\beta_1|
\end{equation}
Since $\lambda_1 > 0$ and $|\beta_1| > 0$, we have:
\begin{equation}
    \bw^\top \dot{\bu} < 0 \quad \forall \bu \in H
\end{equation}
\textbf{Conclusion:} To get to the trap $S_{opp}$, any gradient flow must intersect $H$, but the gradient at the intersection point strictly points away from $S_{opp}$ (the basin of the local minimizer) and towards $S_{same}$. Therefore, the trajectory starting at $\bu_0$ can not reach $S_{opp}$ and will be attract to $S_{same}$. Hence it must converge to the global minimizer $\bs_G$.

\subsection{Proof of Proposition~\ref{prop:eta} (Choice of Step-size)}
\label{subsec:proof-eta}
\textbf{$J$ is a $L$-smooth with $L := 2 \|\Pi_{\mathbf{d}^{\perp}} \mathbf{\Sigma} \Pi_{\mathbf{d}^{\perp}}\|_2$.}

Since $T_{\mathbf{x}}M = \{ \boldsymbol{\xi} \in \mathbb{R}^p : \mathbf{x}^\top \boldsymbol{\xi} = 0, \ \mathbf{d}^\top \boldsymbol{\xi} = 0 \}$, hence $\mathbf{d}^\top \boldsymbol{\xi} = 0$, so $\boldsymbol{\xi} \in \mathbf{d}^{\perp} \implies \Pi_{\bd^\perp} \boldsymbol{\xi}  = \boldsymbol{\xi}.$ Therefore,
\begin{equation} 
\boldsymbol{\xi}^\top \mathbf{\Sigma} \boldsymbol{\xi} = (\Pi_{\mathbf{d}^{\perp}}\boldsymbol{\xi})^\top \mathbf{\Sigma}(\Pi_{\mathbf{d}^{\perp}}\boldsymbol{\xi}) = \boldsymbol{\xi}^\top (\Pi_{\mathbf{d}^{\perp}} \mathbf{\Sigma} \Pi_{\mathbf{d}^{\perp}}) \boldsymbol{\xi}.
\end{equation}
Since $\Pi_{\mathbf{d}^{\perp}} \mathbf{\Sigma} \Pi_{\mathbf{d}^{\perp}}$ is symmetric PSD on $\mathbf{d}^{\perp}$, its Rayleigh quotient is bounded by its spectral norm:
\begin{equation} 
\boldsymbol{\xi}^\top (\Pi_{\mathbf{d}^{\perp}} \mathbf{\Sigma} \Pi_{\mathbf{d}^{\perp}}) \boldsymbol{\xi} \leq \|\Pi_{\mathbf{d}^{\perp}} \mathbf{\Sigma} \Pi_{\mathbf{d}^{\perp}}\|_2 \|\boldsymbol{\xi}\|^2.
\end{equation}

Because $J$ is quadratic, $\nabla^2 J(\mathbf{x}) = 2\mathbf{\Sigma}$, for all $\mathbf{x}.$
Therefore,
\begin{equation}
\label{eq:rayleigh}
\boldsymbol{\xi}^\top \nabla^2 J(\mathbf{x}) \boldsymbol{\xi} = 2 \boldsymbol{\xi}^\top \mathbf{\Sigma} \boldsymbol{\xi} \leq 2 \|\Pi_{\mathbf{d}^{\perp}} \mathbf{\Sigma} \Pi_{\mathbf{d}^{\perp}}\|_2 \|\boldsymbol{\xi}\|^2.
\end{equation}
This proves the claim with $L := 2 \|\Pi_{\mathbf{d}^{\perp}} \mathbf{\Sigma} \Pi_{\mathbf{d}^{\perp}}\|_2$. 

\textbf{Derive the bound for step-size $\eta$.}

Consider the geodesic $\gamma(\tau) := \mathrm{exp}_{\mathbf{x}_t}(\tau \boldsymbol{\xi}_t)$. Since the velocity $\dot{\gamma}$ is tangent to $M$, Eqn.~\ref{eq:rayleigh} applies, bounding the second derivative of the objective $J$ along the path by $\frac{d^2}{d\tau^2}J(\gamma(\tau)) \leq L \|\dot{\gamma}(\tau)\|^2$. This implies the quadratic upper bound:
\[
J(\gamma(\eta)) \leq J(\mathbf{x}_t) + \eta \langle \mathrm{grad } J(\mathbf{x}_t), \boldsymbol{\xi}_t \rangle + \frac{L\eta^2}{2} \|\boldsymbol{\xi}_t\|^2.
\]
Substituting $\boldsymbol{\xi}_t = -\mathrm{grad } J(\mathbf{x}_t)$ and noting $\|\boldsymbol{\xi}_t\| = \|\mathrm{grad } J(\mathbf{x}_t)\|$, we obtain $J(\mathbf{x}_{t+1}) \leq J(\mathbf{x}_t) - \eta(1 - \frac{L\eta}{2}) \|\text{grad } J(\mathbf{x}_t)\|^2$. Since $\eta \leq 1/L$ implies $1 - \frac{L\eta}{2} \geq \frac{1}{2}$, the descent condition holds. 

This proves the Prop.~\ref{prop:eta}

\subsection{Proof of Sublinear Stationarity Rate}
\label{app:sublinear_rate}
\begin{corollary}[Sublinear stationarity rate]
\label{cor:convergence}
Under the assumptions of Prop.~\ref{prop:eta}, for any horizon $T\ge 1$,
\[
\min_{0 \leq t < T} \|\mathrm{grad}\, J(\mathbf{x}_t)\| \le \sqrt{\frac{2(J(\mathbf{x}_0) - J^*)}{\eta T}} = O\!\left(\frac{1}{\sqrt{T}}\right),
\]
where $J^* := \min_{\mathbf{x}\in\mathcal{M}} J(\mathbf{x})$.
\end{corollary}
The proof of Cor.~\ref{cor:convergence} is provided in App.~\ref{subsec:proof-convergence}.

Although the $O(1/\sqrt{T})$ rate is the best guarantee for potentially degenerate landscapes, we can achieve significantly faster local convergence by regularizing the collateral metric.
The proof for Rem.~\ref{remark:linear_convergence} is provided in App.~\ref{subsec:linear-convergence}.
\label{subsec:proof-convergence}
From Proposition~\ref{prop:eta}, the iterates satisfy the descent inequality $J(\mathbf{x}_{t+1}) \leq J(\mathbf{x}_t) - \frac{\eta}{2} \|\mathrm{grad} J(\mathbf{x}_t)\|^2$. Summing this over $t = 0, \dots, T-1$ yields:
\[
J(\mathbf{x}_T) \leq J(\mathbf{x}_0) - \frac{\eta}{2} \sum_{t=0}^{T-1} \|\mathrm{grad} J(\mathbf{x}_t)\|^2.
\]
Noting that $J(\mathbf{x}_T) \geq J^*$, we have $\sum_{t=0}^{T-1} \|\mathrm{grad} J(\mathbf{x}_t)\|^2 \leq \frac{2}{\eta}(J(\mathbf{x}_0) - J^*)$. 
Since minimum of a set is bounded by its average ($\min_t a_t \leq \frac{1}{T}\sum a_t$), we obtain:
\[
\min_{0 \leq t < T} \|\mathrm{grad} J(\mathbf{x}_t)\|^2 \leq \frac{1}{T} \sum_{t=0}^{T-1} \|\mathrm{grad} J(\mathbf{x}_t)\|^2 \leq \frac{2(J(\mathbf{x}_0) - J^*)}{\eta T}.
\]
Taking the square root completes the proof.
\subsection{Proof of Linear Convergence via Regularization}
\begin{remark}[Linear convergence via regularization]
\label{remark:linear_convergence}
The $O(1/\sqrt{T})$ rate is standard when the objective curvature may be flat along some feasible directions, which can occur when $\mathbf{\Sigma}$ is low-rank.
To obtain a local linear convergence rate, one can regularize the objective by adding a small isotropic penalty on $\mathbf{d}^\perp$:
\begin{equation}
\nonumber
    \mathbf{\Sigma}_{\varepsilon} := \mathbf{\Sigma} + \varepsilon \Pi_{\mathbf{d}^{\perp}},
    \qquad \varepsilon>0, \qquad \Pi_{\mathbf{d}^{\perp}} = \mathbf{I} - \mathbf{d}\mathbf{d}^\top.
\end{equation}
\end{remark}

\label{subsec:linear-convergence}

\textbf{Detailed Statemment }

Define the regularized objective
\[
J_{\varepsilon}(\mathbf{x}) := (\mathbf{x} - \mathbf{h})^{\top} \mathbf{\Sigma}_{\varepsilon} (\mathbf{x} - \mathbf{h}), \quad \mathbf{\Sigma}_{\varepsilon} := \mathbf{\Sigma} + \varepsilon \Pi_{\bd^\perp}, \quad \varepsilon > 0,
\]
on the same manifold $M$. Let $\mathbf{x}_{\varepsilon}^*$ be a (global) minimizer of $J_{\varepsilon}$ on $M$.
Assume the Riemannian Hessian of $J_{\varepsilon}$ at $\mathbf{x}_{\varepsilon}^*$ is {positive definite on the tangent space}, i.e., there exists $\mu_{\varepsilon} > 0$ such that
\begin{equation}
\label{eq:hessian-pd}
\langle \boldsymbol{\xi}, \text{Hess}\, J_{\varepsilon}(\mathbf{x}_{\varepsilon}^*)[\boldsymbol{\xi}] \rangle \ge \mu_{\varepsilon} \|\boldsymbol{\xi}\|^2 \quad \forall \boldsymbol{\xi} \in T_{\mathbf{x}_{\varepsilon}^*} M.
\end{equation}
Then there exists a neighborhood $\mathcal{N}$ of $\mathbf{x}_{\varepsilon}^*$ such that if Algorithm 1 is run on $J_{\varepsilon}$ with a sufficiently small step size $\eta \in (0, \bar{\eta}]$ and initialization $\mathbf{x}_0 \in \mathcal{N}$, the iterates satisfy {linear convergence}:
\begin{equation}
\label{eq:linear-conv}
J_{\varepsilon}(\mathbf{x}_t) - J_{\varepsilon}(\mathbf{x}_{\varepsilon}^*) \le (1 - \eta \mu_{\varepsilon})^t \left( J_{\varepsilon}(\mathbf{x}_0) - J_{\varepsilon}(\mathbf{x}_{\varepsilon}^*) \right).
\end{equation}

\textbf{Proof}

Condition~\eqref{eq:hessian-pd} is the Riemannian analogue of \textbf{strong convexity} (restricted to the tangent space). It implies a local Polyak--{\L}ojasiewicz-type inequality in a neighborhood of $\mathbf{x}_{\varepsilon}^*$:
\begin{equation}
\label{eq:pl-inequality}
\|\text{grad}\, J_{\varepsilon}(\mathbf{x})\|^2 \ge 2\mu_{\varepsilon} \left( J_{\varepsilon}(\mathbf{x}) - J_{\varepsilon}(\mathbf{x}_{\varepsilon}^*) \right), \quad \forall \mathbf{x} \in \mathcal{N}.
\end{equation}
On the other hand, $J_{\varepsilon}$ is smooth on $M$, so the same one-step descent bound as before holds locally:
\begin{equation}
\label{eq:descent-bound}
J_{\varepsilon}(\mathbf{x}_{t+1}) \le J_{\varepsilon}(\mathbf{x}_t) - \eta \left( 1 - \frac{L_{\varepsilon}\eta}{2} \right) \|\text{grad}\, J_{\varepsilon}(\mathbf{x}_t)\|^2 \quad \forall \mathbf{x}_t \in \mathcal{N},
\end{equation}
for some local smoothness constant $L_{\varepsilon}$. Choose $\eta \le 1/L_{\varepsilon}$, so the bracket is at least $1/2$, giving
\begin{equation}
\label{eq:descent-simplified}
J_{\varepsilon}(\mathbf{x}_{t+1}) - J_{\varepsilon}(\mathbf{x}_{\varepsilon}^*) \le \left( J_{\varepsilon}(\mathbf{x}_t) - J_{\varepsilon}(\mathbf{x}_{\varepsilon}^*) \right) - \frac{\eta}{2} \|\text{grad}\, J_{\varepsilon}(\mathbf{x}_t)\|^2.
\end{equation}
Finally apply~\eqref{eq:pl-inequality} to~\eqref{eq:descent-simplified}:
\[
J_{\varepsilon}(\mathbf{x}_{t+1}) - J_{\varepsilon}(\mathbf{x}_{\varepsilon}^*) \le (1 - \eta \mu_{\varepsilon}) \left( J_{\varepsilon}(\mathbf{x}_t) - J_{\varepsilon}(\mathbf{x}_{\varepsilon}^*) \right),
\]
and iterate the inequality to obtain~\eqref{eq:linear-conv}.

\section{Derivation of special cases.}
\label{sec:special-cases}
\subsection{Derive Worst-Colateral Damage with $\bd$-Orthogonal Features}
\label{subsec:orthogonal-features}
\textbf{Orthogonal-features assumption.} Other features $\mathbf{f}$ live in the subspace $\mathbf{d}^\perp$. Since we do not know the true feature set, we take a conservative worst-case bound over \textit{all} unit directions in $\mathbf{d}^\perp$:
\begin{equation}
    (\text{collateral bound}) \quad \sup_{\|\mathbf{f}\|=1, \mathbf{f}^\top \mathbf{d}=0} (\mathbf{f}^\top(\mathbf{x}-\mathbf{h}))^2.
\end{equation}

Note that the maximizer $\mathbf{f}$ is a real feature—but serves as a upper bound over other features directions. In this case, since we assume that we might not know exactly where the other features live. To be safe, we will minimize the magnitude of our shift in all directions where a feature could potentially be.

\textbf{Equivalance in objective}
Let $\mathbf{\Delta} := \mathbf{x} - \mathbf{h}$. Split $\mathbf{\Delta}$ into components parallel and orthogonal to $\mathbf{d}$:
    $\mathbf{\Delta} = (\mathbf{d}^\top \mathbf{\Delta})\mathbf{d} + \Pi_{\bd^\perp} \mathbf{\Delta}, \quad \Pi_{\bd^\perp} := \mathbf{I} - \mathbf{d}\mathbf{d}^\top.$

Now take any direction $\mathbf{f}$ with $\mathbf{f}^\top \mathbf{d} = 0$ (so $\mathbf{f} \in \mathbf{d}^\perp$). Then
\begin{equation}
    \mathbf{f}^\top \mathbf{\Delta} = \mathbf{f}^\top ((\mathbf{d}^\top \mathbf{\Delta})\mathbf{d}) + \mathbf{f}^\top (\Pi_{\bd^\perp} \mathbf{\Delta}) = \underbrace{(\mathbf{d}^\top \mathbf{\Delta})(\mathbf{f}^\top \mathbf{d})}_{0} + \mathbf{f}^\top (\Pi_{\bd^\perp} \mathbf{\Delta}) = \mathbf{f}^\top (\Pi_{\bd^\perp} \mathbf{\Delta}).
\end{equation}

So the worst-case bound becomes
\begin{equation}
    \sup_{\|\mathbf{f}\|=1, \mathbf{f}^\top \mathbf{d}=0} (\mathbf{f}^\top \mathbf{\Delta})^2 = \sup_{\|\mathbf{f}\|=1, \mathbf{f} \in \mathbf{d}^\perp} (\mathbf{f}^\top (\Pi_{\bd^\perp} \mathbf{\Delta}))^2.
\end{equation}

Apply Cauchy–Schwarz in the subspace $\mathbf{d}^\perp$:
\begin{equation}
    |\mathbf{f}^\top (\Pi_{\bd^\perp} \mathbf{\Delta})| \le \|\mathbf{f}\| \|\Pi_{\bd^\perp} \mathbf{\Delta}\| = \|\Pi_{\bd^\perp} \mathbf{\Delta}\|.
\end{equation}

Therefore,
\begin{equation}
    \sup_{\|\mathbf{f}\|=1, \mathbf{f}^\top \mathbf{d}=0} (\mathbf{f}^\top (\mathbf{x}-\mathbf{h}))^2 = \|\Pi_{\bd^\perp} (\mathbf{x}-\mathbf{h})\|^2.
\end{equation}

So the simplified objective becomes
\begin{equation}
     \min_{\mathbf{x} \in \mathcal{M}, ||\bx|| = 1, \bx^\top\bd = \alpha} \|\Pi_{\bd^\perp} (\mathbf{x}-\mathbf{h})\|^2.
\end{equation}
\textbf{Solution}
Use the parameterization $\mathbf{x} = \alpha \mathbf{d} + r\mathbf{u}$ with $\mathbf{u} \perp \mathbf{d}, \|\mathbf{u}\|=1$. Let $\mathbf{h}_\perp := \Pi_{\bd^\perp} \mathbf{h} = \mathbf{h} - (\mathbf{d}^\top \mathbf{h})\mathbf{d}$. 

The objective simplifies as follows:
\begin{equation}
    \|\Pi_{\bd^\perp} (\mathbf{x}-\mathbf{h})\|^2 = \|r\mathbf{u} - \mathbf{h}_\perp\|^2
\end{equation}
\begin{equation}
    = r^2 \|\mathbf{u}\|^2 + \|\mathbf{h}_\perp\|^2 - 2r \mathbf{u}^\top \mathbf{h}_\perp = r^2 + \|\mathbf{h}_\perp\|^2 - 2r \mathbf{u}^\top \mathbf{h}_\perp
\end{equation}

Minimizing this is equivalent to maximizing $\mathbf{u}^\top \mathbf{h}_\perp$. By Cauchy–Schwarz, the maximum is reached when $\mathbf{u}$ aligns with $\mathbf{h}_\perp$:
\begin{equation}
    \mathbf{u}^* = \frac{\mathbf{h}_\perp}{\|\mathbf{h}_\perp\|} \quad (\text{if } \mathbf{h}_\perp \neq \mathbf{0})
\end{equation}

The optimal solution is thus:
\begin{equation}
    \mathbf{x}^* = \alpha \mathbf{d} + r \frac{\mathbf{h} - (\mathbf{d}^\top \mathbf{h})\mathbf{d}}{\|\mathbf{h} - (\mathbf{d}^\top \mathbf{h})\mathbf{d}\|}
\end{equation}

\subsection{Derive Worst-Collateral Damage with $\bd$-Near-Orthogonal Features}
\label{subsec:near-orthogonal}

\textbf{Near-orthogonal other features.} Instead of requiring $\bbf^\top \bd = 0$, a small overlap is allowed:
\begin{equation}
    \|\bbf\| = 1, \quad |\bbf^\top \bd| \le \varepsilon \text{,}
\end{equation}
where $\varepsilon \in [0, 1)$ is small.

\textbf{Worst-case collateral objective (same style as the remark):}
\begin{equation}
    \min_{\bx \in \mathcal{M}} \sup_{\|\bbf\|=1, |\bbf^\top \bd| \le \varepsilon} (\bbf^\top (\bx - \bh))^2 \text{.}
\end{equation}
As before, this is a worst-case bound over directions; it is not assumed that the maximizer $\bbf$ is an actual feature.

\textbf{Equivalence of objective function}
Similar to the $\bd$-orthogonal case, $\mathbf{\Delta} := \bx - \bh$. Decompose $\mathbf{\Delta}$ into parts parallel and orthogonal to $\bd$:
\begin{itemize}
    \item parallel scalar: $\beta := \bd^\top \mathbf{\Delta} = \bd^\top \bx - \bd^\top \bh = \alpha - \bd^\top \bh$.
    \item orthogonal part: $\mathbf{\Delta}_\perp := \Pi_{\bd^\perp} \mathbf{\Delta} = \Pi_{\bd^\perp} (\bx - \bh)$.
\end{itemize}
Important point: because $\bx \in \mathcal{M}$ forces $\bd^\top \bx = \alpha$, the scalar $\beta$ is \textbf{constant over all feasible $\bx$}.

Taking any feasible $\bbf$ with $\|\bbf\| = 1$ and $|\bbf^\top \bd| \le \varepsilon$, we denote: $a(\bbf) := \bbf^\top \bd \quad (|a(\bbf)| \le \varepsilon)$. Then $\bbf$ can be decomposed as:
\begin{equation}
    \bbf = a(\bbf) \bd + \sqrt{1 - a(\bbf)^2} \bu_f, \quad \text{where } \bu_f \perp \bd, \|\bu_f\| = 1 \text{.}
\end{equation}
Therefore:
\begin{equation}
    \bbf^\top \mathbf{\Delta} = a(\bbf)(\bd^\top \mathbf{\Delta}) + \sqrt{1 - a(\bbf)^2} \bu_f^\top \mathbf{\Delta}_\perp = a(\bbf)\beta + \sqrt{1 - a(\bbf)^2} \bu_f^\top \mathbf{\Delta}_\perp \text{.}
\end{equation}
For any fixed $a$, the worst case over $\bu_f$ is achieved by aligning $\bu_f$ with $\mathbf{\Delta}_\perp$, giving:
\begin{equation}
    \sup_{\bu_f \perp \bd, \|\bu_f\|=1} |\bbf^\top \mathbf{\Delta}| = |a| |\beta| + \sqrt{1 - a^2} \|\mathbf{\Delta}_\perp\| \text{.}
\end{equation}
Therefore the whole supremum becomes a maximization over $a$:
\begin{equation}
    \sup_{\|\bbf\|=1, |\bbf^\top \bd| \le \varepsilon} |\bbf^\top \mathbf{\Delta}| = \max_{0 \le a \le \varepsilon} \left( a|\beta| + \sqrt{1 - a^2} \|\mathbf{\Delta}_\perp\| \right) \text{.}
\end{equation}
Since $\sup (\bbf^\top \mathbf{\Delta})^2 = (\sup |\bbf^\top \mathbf{\Delta}|)^2$, the $\bd$-near-orthogonal objective becomes:
\begin{equation}
    \min_{\bx \in \mathcal{M}} \left[ \max_{0 \le a \le \varepsilon} \left( a|\beta| + \sqrt{1 - a^2} \|\Pi_{\bd^\perp} (\bx - \bh)\| \right) \right]^2 \text{.}
\end{equation}
Since $|\beta| = |\alpha - \bd^\top \bh|$ is a constant, the only $\bx$-dependence component in the objective is:
\begin{equation}
    A(\bx) := \|\Pi_{\bd^\perp} (\bx - \bh)\|.
\end{equation}
\textbf{Solution}
\textit{Monotonicity fact.} For any fixed $a \in [0, \varepsilon]$, the expression $a|\beta| + \sqrt{1 - a^2} A(\bx)$ is an increasing function of $A(\bx)$. A maximum of increasing functions is also increasing, and squaring preserves the argmin since everything is nonnegative. Hence, minimizing the objective over $\bx \in \mathcal{M}$ is equivalent to minimizing $A(\bx) = \|\Pi_{\bd^\perp} (\bx - \bh)\|$ over $\bx \in \mathcal{M}$, for any fixed $a$.

That means the near-orthogonal problem has exactly the same minimizers in Eqn.~\ref{eq:slerp} as the orthogonal case $\min_{\bx \in \mathcal{M}} \|\Pi_{\bd^\perp} (\bx - \bh)\|^2$.

\subsection{Uniform features on a sphere}
\label{subsec:uniform-sphere}
Instead of the conservative worst-case over $\bbf$, we assume the collateral direction $\bbf$ is random and uniform on the unit sphere in $\mathbb{R}^p$. The expected collateral damage objective is:
\[
\min_{\bx \in \mathcal{M}} \mathbb{E}_{\bbf} \left[ (\bbf^\top (\bx - \bh))^2 \right].
\]
The expectation expands as:
$\mathbb{E}[(\bbf^\top \mathbf{\Delta})^2] = \mathbb{E}[\mathbf{\Delta}^\top (\bbf\bbf^\top) \mathbf{\Delta}] = \mathbf{\Delta}^\top \mathbb{E}[\bbf\bbf^\top] \mathbf{\Delta}.$
Since the uniform distribution on a sphere is rotationally invariant, the second moment matrix is isotropic:
\[
\mathbb{E}[\bbf\bbf^\top] = \frac{1}{p} \bI.
\]
Substituting this back, the objective becomes minimizing the squared Euclidean distance:
\[
\mathbb{E}[(\bbf^\top (\bx - \bh))^2] = \frac{1}{p} \|\bx - \bh\|^2.
\]
Again, since $\bd^\top (\bx - \bh) = \alpha - \bd^\top \bh$ is a constant and
\[
\|\bx - \bh\|^2 = \|\Pi_{\bd^\perp} (\bx - \bh)\|^2 + (\bd^\top (\bx - \bh))^2,
\]
Hence, minimizing the total distance $\|\bx - \bh\|^2$ over $\mathcal{M}$ is exactly equivalent to minimizing $\|\Pi_{\bd^\perp} (\bx - \bh)\|^2$, yielding the unique SLERP point in $\text{span}\{\bh, \bd\}$ as the solution.

\section{More experiment details}
\label{app:exp_details}
\subsection{Models.}
\label{app:Models}
We evaluate across a diverse set of instruction-tuned LLMs spanning model families and scales:
Llama-3.2-3B-Instruct~\cite{dubey2024llama}, Llama-3.1-8B-Instruct~\cite{dubey2024llama}, Qwen2.5-14B-Instruct~\cite{qvq-72b-preview}, and Gemma-2-9B-it~\cite{team2024gemma}.
All results are reported using the official tokenizers and checkpointed weights.

\subsection{Intervention locations.}
\label{app:intervention_locations}
We perform activation steering by editing hidden states at \textit{intervention locations}, positioned after the layer normalization of the attention and MLP layers, results in $2L$ locations, where $L$ is the number of model's layers. We apply this setting to all methods for a fair comparison.

\subsection{Target direction construction.}
\label{app:steering_vector_construction}
We precompute the steering direction $\mathbf{d}_{\ell}$ for each intervention location using a calibration dataset following the Angular Steering protocol. We define $\mathcal{D}_{\mathrm{harmful}}$ using an 80\% split of AdvBench~\cite{zou2023advbench} with 416 harmful instructions and $\mathcal{D}_{\mathrm{harmless}}$ using randomly sampled 512 harmless instructions from Alpaca~\cite{alpaca}.

Let $\mathbf{h}_{\ell,t}(x) \in \mathbb{R}^d$ be the activation at layer $\ell$ and token position $t$, and let $\hat{\mathbf{h}}_{\ell,t}(x) = \mathbf{h}_{\ell,t}(x) / \|\mathbf{h}_{\ell,t}(x)\|$ denote its unit-normalized version. Here, we omit sub-layer indices for simplicity. We compute the normalized mean activations for each class:
$\mu_{\mathrm{harmful}} = \frac{1}{|\mathcal{D}_{\mathrm{harmful}}|}\sum_{x\in\mathcal{D}_{\mathrm{harmful}}} \hat \bh_{\ell,t}(x),
\qquad
\mu_{\mathrm{harmless}} = \frac{1}{|\mathcal{D}_{\mathrm{harmless}}|}\sum_{x\in\mathcal{D}_{\mathrm{harmless}}} \hat \bh_{\ell,t}(x),$
The feature direction is then defined by the standard difference-in-means estimator~\cite{belrose2023diff-in-means}, $\mathbf{d}_{\ell} \propto \mu_{\mathrm{harmful}} - \mu_{\mathrm{harmless}}$, renormalized to unit length. This yields a different steering direction for each of the $2L$ locations. We use these intervention location-specific directions, whereas for Angular Steering, we keep their selection rule for a single direction to apply across all locations.

\subsection{Collateral matrix estimation and optimization hyperparameters.}
\label{app:sigma_compute}
We precompute the collateral matrix using the activation second moment, denoted $\mathbf{\Sigma}_{\mathbf{h_{\text{ref}}},\ell}$ for layer $\ell$ (omitting sub-layer indices). For each intervention location, we collect $N=100{,}000$ token activations from the C4 corpus~\cite{raffel2020exploring}. Applying the same unit-normalization as in the target direction construction, we compute the empirical second moment at layer $\ell$ as $\mathbf{\Sigma}_{\mathbf{h_{\text{ref}}},\ell} = \frac{1}{N}\sum_{i=1}^{N}\hat{\mathbf{h_{\text{ref}}}}_{\ell, i}\hat{\mathbf{h_{\text{ref}}}}_{\ell, i}^{\top}.$ 
While we use C4 here, other corpus can be chosen to reflect specific task priorities or user preferences, as it defines which feature directions are costly to perturb. 

To ensure scale insensitivity across layers and models, we normalize the matrix by its top eigenvalue, $\mathbf{\Sigma}_{\mathbf{h},\ell} \leftarrow \mathbf{\Sigma}_{\mathbf{h_{\text{ref}}},\ell} / \lambda_{\max}(\mathbf{\Sigma}_{\mathbf{h_{\text{ref}}},\ell})$. Since this ensures $\|\mathbf{\Sigma}_{\mathbf{h_{\text{ref}}},\ell}\|_2=1$, Prop.~\ref{prop:eta} implies that any step size $\eta \in (0, 0.5]$ guarantees objective decrease. In all experiments, we set $\eta=0.3$ and use $T=1$ iteration for efficiency, finding that a single geodesic step is sufficient for strong performance.

\subsection{Baselines.}
\label{app:baselines}
We compare our method to the following steering operators:
(i) {SLERP}, which interpolates on the unit sphere toward the target direction;
(ii) {Angular Steering}, which rotates activations within a fixed 2D steering plane to a target in-plane angle;
and (iii) {ActAdd}, with a scalar coefficient $\lambda$ controlling strength, and (iv) {No steering}.

\subsection{Steering strength.}
\label{app:steering_strength}
Different baselines expose different knobs: SLERP/COAST naturally accept an alignment target expressed as a cosine
similarity with $\bd_{\ell}$, ActAdd uses an additive coefficient, and Angular Steering controls a rotation angle in a
chosen 2D plane. To enable comparison with angular steering by degree of steered, we match methods using a common \emph{angular strength} parameter. Specifically, we sweep an angle budget $\theta\in[0^\circ,180^\circ]$ and define the corresponding cosine target $\alpha = \cos(\theta)$.
Restricting $\theta$ to $[0^\circ,180^\circ]$ avoids cosine aliasing on $[0^\circ,360^\circ]$ (distinct angles can share the
same cosine), while spanning the full range of cosine targets from $+1$ to $-1$.  For Angular Steering, the method control steering directly by a target in-plane angle after projecting onto the steering plane. For actadd, we steer with coefficient from -10 to 10, each interval of 1.

\subsection{Adaptive Alignment Budget.}
\label{app:aab}
A single global target $\alpha$ applied uniformly to all tokens and layers can be overly strict: tokens whose pre-steering activation is nearly orthogonal to $\bd_{\ell}$ would be forced to undergo a large rotation, which can introduce unnecessary collateral change. To make steering more context-sensitive while keeping a single user-facing control parameter, we set the
constraint for SLERP/COAST at each $(\ell,t)$ as
$\alpha_{\ell,t} \;=\; \alpha \cdot \bigl|\langle \bh_{\ell,t}, \bd_{\ell} \rangle\bigr|$ where $\bh_{\ell,t}=\bh_{\ell,t}/\|\bh_{\ell,t}\|$.
Intuitively, this allocates more steering budget to states that already express the target feature (large $|\langle \bh_{\ell,t}, \bd_{\ell} \rangle|$) and imposes a tighter constraint on unrelated states, reducing unnecessary
global distortion.

\subsection{Norm preservation for spherical steering (SLERP/COAST).}
\label{app:norm_preserve}
SLERP and COAST operate on the unit sphere to preserve directionality while controlling angular displacement.
To preserve the original activation scale, we implement both methods by (1) normalizing the activation,
(2) performing the spherical steering update on $\hat h_{\ell,t}$ to obtain $\hat x_{\ell,t}$, and then
(3) rescaling back to the original norm: $x_{\ell,t} \;=\; \|h_{\ell,t}\| \cdot \hat x_{\ell,t}$.

\section{Computational Cost Analysis.}
\label{subsec:computational-cost}
To validate the practical feasibility of our approach, we evaluate the inference throughput (tokens per second) across all methods using the official implementations. As reported in Table~\ref{tab:throughput}, the overhead introduced by steering is minimal. While COAST involves slight additional computational cost due to the per-token normalization and rescaling operations—resulting in a throughput reduction of approximately 0.9\% to 4.0\% compared to the unsteered baseline—this overhead is small in practice. The method maintains high-speed generation, confirming that COAST is well-suited for real-time deployment where both safety and low latency are critical.
\begin{table*}[t]
\centering
\caption{Inference throughput (tokens/sec) and relative generation speed change compared to the unsteered baseline. Negative percentages indicate a reduction in speed.}
\label{tab:throughput}
\resizebox{\textwidth}{!}{%
\begin{tabular}{l c cc cc cc cc}
\toprule
\textbf{Model} & \textbf{No Steering} & \multicolumn{2}{c}{\textbf{Angular}} & \multicolumn{2}{c}{\textbf{ActAdd}} & \multicolumn{2}{c}{\textbf{SLERP}} & \multicolumn{2}{c}{\textbf{COAST}} \\
 & (tok/s) & (tok/s) & ($\Delta$\%) & (tok/s) & ($\Delta$\%) & (tok/s) & ($\Delta$\%) & (tok/s) & ($\Delta$\%) \\
\midrule
Gemma-2-9B-It & 2364.4 & 2360.1 & \small{-0.18\%} & 2360.5 & \small{-0.16\%} & 2352.7 & \small{-0.49\%} & 2342.2 & \small{-0.94\%} \\
Llama-3.2-3B-Instruct & 5897.9 & 5744.7 & \small{-2.60\%} & 5820.1 & \small{-1.32\%} & 5745.4 & \small{-2.59\%} & 5663.6 & \small{-3.97\%} \\
Llama-3.1-8B-Instruct & 3771.0 & 3730.4 & \small{-1.08\%} & 3754.4 & \small{-0.44\%} & 3703.6 & \small{-1.79\%} & 3673.8 & \small{-2.58\%} \\
Qwen2.5-14B-Instruct & 2521.1 & 2511.4 & \small{-0.38\%} & 2522.6 & \small{+0.06\%} & 2505.4 & \small{-0.62\%} & 2467.6 & \small{-2.12\%} \\
\bottomrule
\end{tabular}%
}
\end{table*}

\section{Solution for Collateral-damage Minimizing Activation Steering via KKT Conditions}
\label{sec:kkt-solution}

\textbf{Problem Formulation Recall.} Given an activation vector $\bh \in \mathbb{R}^p$, a steering direction $\bd \in \mathbb{R}^p$ with $\|\bd\|=1$, and a target alignment $\alpha \in (\bd^\top \bh, 1]$, we seek a steered activation $\bx$ that minimizes collateral damage while achieving the desired alignment:
\begin{equation}
\begin{aligned} \label{app:optimization_problem}
\min_{\bx} \quad & J(\bx) = (\bx-\bh)^\top \bSigma (\bx-\bh) \\
\textrm{s.t.} \quad & \|\bx\|^2 = 1 \\
& \bd^\top \bx = \alpha
\end{aligned}
\end{equation}
where $\bSigma = \mathbb{E}[\bh\bh^\top]$ is the empirical second-moment matrix of activations, computed from a reference corpus. The objective $J(\bx)$ measures the expected squared change along feature directions, weighted by the activation covariance. The constraints ensure unit norm (preserving activation magnitude) and target alignment with the steering direction. 

\textbf{Closed-Form Solution via KKT Conditions.} In the main paper, we solve~\eqref{app:optimization_problem} using Riemannian gradient descent (Alg.~\ref{alg:geodesic_descent}). Here we derive an equivalent closed-form solution via the KKT conditions. The Lagrangian with multipliers $\lambda$ and $\mu$ for the norm and alignment constraints is:
\begin{equation}
\mathcal{L}(\bx, \lambda, \mu) = (\bx - \bh)^\top \bSigma (\bx - \bh) + \lambda(\|\bx\|^2 - 1) + \mu(\bd^\top \bx - \alpha)
\end{equation}
Setting $\nabla_{\bx} \mathcal{L} = 2\bSigma(\bx - \bh) + 2\lambda \bx + \mu \bd = 0$ and rearranging:
\begin{equation}
\label{eq:kkt-x-solution}
\bx(\lambda, \mu) = (\bSigma + \lambda \bI)^{-1}\left(\bSigma \bh - \frac{\mu}{2}\bd\right)
\end{equation}

To find $\bx$, it remains to determine the Lagrange multipliers $\lambda$ and $\mu$. We can eliminate $\mu$ by using the alignment constraint. Let $\bSigma = \bQ \boldsymbol{\Lambda} \bQ^\top$ be the eigendecomposition with $\boldsymbol{\Lambda} = \text{diag}(\sigma_1, \dots, \sigma_p)$, and define $\tilde{\bh} = \bQ^\top \bh$ and $\tilde{\bd} = \bQ^\top \bd$ as the representations in the eigenbasis. Substituting~\eqref{eq:kkt-x-solution} into $\bd^\top \bx = \alpha$ and solving for $\mu$ results in 
\begin{equation}
\label{eq:kkt-mu}
\mu(\lambda) = \frac{2(A(\lambda) - \alpha)}{B(\lambda)},
\end{equation}
where 
\begin{equation}
A(\lambda) = \sum_{i=1}^{p} \frac{\sigma_i \tilde{d}_i \tilde{h}_i}{\sigma_i + \lambda}, \quad B(\lambda) = \sum_{i=1}^{p} \frac{\tilde{d}_i^2}{\sigma_i + \lambda}.
\end{equation}

Substituting $\mu(\lambda)$ back into~\eqref{eq:kkt-x-solution} and applying the norm constraint $\|\bx\|^2 = 1$, we obtain a scalar equation in $\lambda$:
\begin{equation}
\label{eq:kkt-G}
G(\lambda) = \sum_{i=1}^{p} \left(\frac{\sigma_i \tilde{h}_i - \frac{\mu(\lambda)}{2}\tilde{d}_i}{\sigma_i + \lambda}\right)^2 - 1 = 0
\end{equation}

Solving $G(\lambda) = 0$ for $\lambda^*$ then gives $\mu^*$ via~\eqref{eq:kkt-mu} and $\bx$ via~\eqref{eq:kkt-x-solution}. Standard root-finding methods such as Brent's or Newton's method can be used. The eigendecomposition of $\bSigma$ can be precomputed once per layer, and for each $\bh$, the cost reduces to two matrix-vector products with $\bQ$ and a one-dimensional search.

\end{document}